\documentclass[final,journal]{IEEEtran}
\usepackage{optidef}

\newcommand{\avec}{{\bf{a}}}

\newcommand{\evec}{{\bf{e}}}

\newcommand{\pvec}{{\bf{p}}}

\newcommand{\xvec}{{\bf{x}}}
\newcommand{\zvec}{{\bf{z}}}

\newcommand{\gvec}{{\bf{g}}}

\newcommand{\hvec}{{\bf{h}}}

\newcommand{\etavec}{{\bf{\eta}}}

\newcommand{\onevec}{{\bf{1}}}
\newcommand{\zerovec}{{\bf{0}}}

\newcommand{\muvec}{{\bf{\mu}}}
\newcommand{\alphavec}{{\bf{\alpha}}}
\newcommand{\Zeromat}{{\bf{0}}}

\newcommand{\Lambdamat}{{\bf{\Lambda}}}

\newcommand{\Thetamat}{{\bf{\Theta}}}
\newcommand{\Amat}{{\bf{A}}}
\newcommand{\Bmat}{{\bf{B}}}
\newcommand{\Cmat}{{\bf{C}}}
\newcommand{\Dmat}{{\bf{D}}}
\newcommand{\Emat}{{\bf{E}}}

\newcommand{\Hmat}{{\bf{H}}}
\newcommand{\Jmat}{{\bf{J}}}
\newcommand{\Imat}{{\bf{I}}}
\newcommand{\Lmat}{{\bf{L}}}

\newcommand{\Umat}{{\bf{U}}}
\newcommand{\Pmat}{{\bf{P}}}

\newcommand{\Smat}{{\bf{S}}}

\newcommand{\Rmat}{{\bf{R}}}
\newcommand{\Vmat}{{\bf{V}}}

\newcommand{\define}{\stackrel{\triangle}{=}}

\newcommand{\Psimat}{\mbox{\boldmath $\Psi$}}

\def\btheta{{\mbox{\boldmath $\theta$}}}

\def\btheta{{\mbox{\boldmath $\theta$}}}

\def\alphavec{{\mbox{\boldmath $\alpha$}}}

\def\etavec{{\mbox{\boldmath $\eta$}}}

\def\thetavec{{\mbox{\boldmath $\theta$}}}

\def\muvec{{\mbox{\boldmath $\mu$}}}

\def\thetavecsmall{{\mbox{\boldmath {\scriptsize $\theta$}}}}
\def\alphavecsmall{{\mbox{\boldmath {\scriptsize $\alpha$}}}}

\newcommand{\be}{\begin{equation}}
\newcommand{\ee}{\end{equation}}
\newcommand{\beqna}{\begin{eqnarray}}
\newcommand{\eeqna}{\end{eqnarray}}

\usepackage{amsmath,amssymb,amsthm,enumerate}
\usepackage{float}
\usepackage{mathtools}
\usepackage{multirow,graphicx}
\usepackage{cite}
\usepackage{graphicx}
\usepackage{nicefrac}
\usepackage{mdframed}
\usepackage[ruled,vlined]{algorithm2e}
\usepackage{dsfont} 
\usepackage{tikz}
\usepackage{xcolor}
\newtheorem{theorem}{Theorem}

\usepackage[bookmarks, colorlinks=true, linkcolor=red, citecolor=green, urlcolor=blue]{hyperref}

\usepackage{epstopdf}

\usepackage[printonlyused]{acronym}
\acrodef{mle}[MLE]{maximum likelihood estimator}
\acrodef{cmle}[CMLE]{constrained maximum likelihood estimator}
\acrodef{crb}[CRB]{Cram$\acute{\text{e}}$r-Rao Bound}
\acrodef{ccrb}[CCRB]{constrained Cram$\acute{\text{e}}$r-Rao Bound}
\acrodef{dlpf}[DLPF]{decoupled linear power flow}
\acrodef{ac}[AC]{alternating current}
\acrodef{dc}[DC]{direct current}
\acrodef{alm}[ALM]{augmented Lagrangian method}
\acrodef{admm}[ADMM]{alternating directional method of multipliers}
\acrodef{pgd}[PGD]{projected gradient descent}
\acrodef{algoadmm}[ADMM]{Alternating Directional Method of Multipliers}
\acrodef{tls}[TLS]{total least squares}
\acrodef{wrt}[w.r.t.]{with respect to}
\acrodef{pdf}[PDF]{probability density function}
\acrodef{fim}[FIM]{Fisher Information Matrix}
\acrodef{pmu}[PMUs]{phasor measurement units}
\acrodef{mse}[MSE]{mean-squared-error}
\acrodef{ols}[OLS]{Ordinary Least Squares}
\acrodef{gsp}[GSP]{graph signal processing}
\acrodef{gso}[GSO]{graph shift operator}
\acrodef{snr}[SNR]{signal-to-noise ratio}
\acrodef{gmrf}[GMRF]{Gaussian Markov random field}
\acrodef{lu-ccrb}[LU-CCRB]{Lehmann-unbiased CCRB}
\acrodef{glasso}[GLASSO]{graphical LASSO}
\acrodef{lgmrf}[LGMRF]{Laplacian-constrained GMRF}
\acrodef{re}[RE]{Relative Error}
\acrodef{alpe}[ALPE]{Adaptive Laplacian-constrained Precision matrix Estimation}
\acrodef{newgle}[NewGLE]{New Graph Laplacian Estimation}
\acrodef{psd}[PSD]{positive semi-definite}

\acrodef{iid}[i.i.d.]{independent and identically distributed}
\acrodef{dn}[DN]{distribution network}
\acrodef{tn}[TN]{transmission network}
\DeclareMathAlphabet{\pazocal}{OMS}{zplm}{m}{n} 
\acrodef{gsp}[GSP]{graph signal processing}
\acrodef{gft}[GFT]{Graph Fourier Transform}

\definecolor{mydarkgreen}{rgb}{0.0, 0.5, 0.0}
\definecolor{mydarkred}{rgb}{0.5, 0.0, 0.0}
\definecolor{mydarkblue}{rgb}{0.0, 0.0, 0.5}
\definecolor{myblue}{RGB}{70, 130, 180}
\definecolor{myred}{RGB}{178, 34, 34}

\usetikzlibrary{arrows.meta, positioning}

\SetKwInput{KwInput}{Input}
\SetKwInput{KwOutput}{Output}

\usepackage{bbm}
\usepackage{xcolor}

\ifCLASSINFOpdf

\else

\fi

\SetKwInput{KwInput}{Input}                
\SetKwInput{KwOutput}{Output}              
\SetKwInput{KwInit}{Initialization}        

\graphicspath{ {./Figures/} }

\usepackage[most]{tcolorbox}
\newtcolorbox{boxA}[2][]{%
  attach boxed title to top center
              = {yshift=-8pt},
  colback      = brown!30,
  colframe     = brown!80!black,
  fonttitle    = \bfseries\color{brown!70!black},
  colbacktitle = orange!30!brown,
  title        = #2,#1,
  enhanced,
}

\begin{document}

\title{Cram$\acute{\text{e}}$r-Rao Bounds for Laplacian Matrix Estimation
}
\author{Morad Halihal,
Tirza Routtenberg, \IEEEmembership{Senior Member, IEEE},
and  H. Vincent Poor \IEEEmembership{Fellow Member, IEEE}

\thanks{
{\footnotesize{M. Halihal and T. Routtenberg are with the School of Electrical and Computer Engineering, Ben-Gurion University of the Negev, Beer-Sheva 84105, Israel, e-mail: moradha@post.bgu.ac.il, tirzar@bgu.ac.il.
 H. V. Poor
is with the Department of Electrical and Computer Engineering, Princeton University, Princeton, NJ, e-mail: poor@princeton.edu.
}}}
\thanks{ Parts of this work were presented at the IEEE International Conference on Acoustics, Speech, and Signal Processing (ICASSP) 2024 as the paper~\cite{ICASSP2024MORAD}. This research was supported by the ISRAEL SCIENCE FOUNDATION (Grant No. 1148/22) and 
by the Israel Ministry of National
Infrastructure, Energy, and Water Resources.}
}
\maketitle

\begin{abstract}
In this paper, we analyze the performance of the estimation of Laplacian matrices under general observation models. Laplacian matrix estimation involves structural constraints, including symmetry and null-space properties, along with matrix sparsity.
 By exploiting a linear reparametrization that enforces the structural constraints,
we derive closed-form matrix expressions for the Cram$\acute{\text{e}}$r-Rao Bound (CRB)  specifically tailored to Laplacian matrix estimation. We further extend the derivation to the sparsity-constrained case,
introducing two oracle CRBs that incorporate prior information of the support set, i.e. the locations of the nonzero entries in the Laplacian matrix. We examine the properties and order relations between the bounds, and provide the associated Slepian-Bangs formula for the Gaussian case. 
We demonstrate the use of the new CRBs in three representative applications: (i) topology identification in power systems, (ii) graph filter identification in diffused models, and (iii) precision matrix estimation in Gaussian Markov random fields under Laplacian constraints. 
The CRBs are evaluated and compared with the mean-squared-errors (MSEs) of the constrained maximum likelihood estimator (CMLE), which integrates both equality and inequality constraints along with sparsity constraints,
and of the oracle CMLE, which knows the locations of the nonzero entries of the Laplacian matrix. 
We perform this analysis for the applications of power system topology identification and graphical LASSO, and demonstrate that the MSEs of the estimators converge to the CRB and oracle CRB, given a sufficient number of measurements.

\end{abstract}
\begin{IEEEkeywords}
Graph signal processing (GSP), Laplacian matrix estimation, Cram$\acute{\text{e}}$r-Rao bound (CRB), sparsity constraints
\end{IEEEkeywords}


\section{Introduction}
\label{sec:intro}

Graph-structured data and signals arise in numerous applications, including power systems, communications,  finance, social networks, and biological networks, for analysis and inference of networks \cite{sandryhaila2014discrete, shuman2013emerging}.
In this context, the Laplacian matrix, which captures node connectivity and edge weights, serves as a fundamental tool for clustering \cite{cardoso2020algorithmslearninggraphsfinancial},  modeling graph diffusion processes \cite{Shafipour2021Diffusion,pasdeloup2017characterization}, topology inference \cite{halihal2024estimation,segarra2017network,dong2016learning,mei2015signal,pasdeloup2017characterization,chepuri2017learning,kalofolias2016learn}, anomaly detection \cite{drayer2020detection}, graph-based filtering \cite{dabush2021state,kroizer2022bayesian,ramakrishna2020user,he2022detecting,isufi2024graph},
and analyzing smoothness on graphs \cite{Lital2023Smooth}. In the emerging field of \ac{gsp}, the Laplacian matrix enables fundamental operations such as filtering and interpolation on graphs.
Thus,  Laplacian matrix estimation is crucial across various domains.

As a motivating example, a key real-world application of Laplacian matrix estimation is in power systems, where the admittance matrix plays a crucial role in grid monitoring, optimization, and cyber security \cite{Abur_Gomez_book, Giannakis_Wollenberg_2013}. Hence, accurate estimation of this matrix is essential for topology identification, line parameter estimation, outage identification \cite{Poor_Tajer_2012}, and cyber attack detection \cite{drayer2020detection, Gal2023detection}. The increasing integration of distributed renewable energy resources further amplifies the need for robust admittance matrix estimation and its associated performance analysis \cite{ICASSP2024MORAD,zhang2023topology}.
Accurate performance analysis is therefore essential for the design, monitoring, and reliable operation of modern power systems.


Numerous approaches have been proposed in the literature for the estimation of the Laplacian matrix. In \ac{gsp}, learning the graph Laplacian matrix from smooth graph signals has been studied in \cite{Xiaowen2015Laplace, Vassilis2016graph}. In power systems, admittance matrix estimation has been explored in \cite{GridGSP,Li_Poor_Scaglione_2013,Park_Deka_Chertkov2018,Morad_ICASSP2022,halihal2024estimation,Grotas_Routtenberg_2019}.
In graphical models \cite{10.1093/biostatistics/kxm045,Egilmez_Pavez_Ortega_2017,dong2016learning,ying2020nonconvex,zhao2019optimization,pavez2016generalized,dAspremont2008first}, the seminal work in \cite{10.1093/biostatistics/kxm045} introduced the \ac{glasso} algorithm, a regularization-based framework for  sparse precision (inverse covariance) matrix estimation. 
Various computationally efficient extensions of \ac{glasso} have been proposed (see, e.g. \cite{Banerjee2008SparseMLE, Mazumder2011TheGL}). 
More recently, estimation methods that enforce Laplacian structure on the precision matrix have gained attention \cite{Egilmez_Pavez_Ortega_2017, ying2021minimax, Yakov2023Laplace}. 
However, despite these advances, existing estimation approaches lack theoretical performance guarantees, which are essential for understanding their limitations, guiding system design, and assessing estimator performance.

Performance bounds are widely used for system design and performance analysis in estimation theory.  
The well-known \ac{crb} is a lower bound on the variance of any unbiased estimator and serves as a benchmark for assessing estimation accuracy.
The estimation of the Laplacian matrix involves additional parametric constraints in the form of equality, inequality, and sparsity constraints \cite{Morad_ICASSP2022,halihal2024estimation}, necessitating performance bounds that incorporate these structural properties.
The \ac{ccrb} \cite{gorman1990lower,Stoica_Ng, Nitzan_constraints} is a lower bound on the \ac{mse} of any $\chi$-unbiased estimator \cite{benhaim2010cramer, Nitzan_constraints}, which integrates equality parametric constraints to obtain a valid bound. Moreover, the \ac{ccrb} is asymptotically attained by the commonly used \ac{cmle} \cite{osborne_2000,benhaim2010cramer}.
Since linear constraints can be equivalently handled via reparametrization \cite{Moore_scoring, menni2014versatility}, an alternative approach is to derive the conventional \ac{crb} under a reduced-dimensional reparametrization, which is the approach adopted in this paper. 
In addition, sparsity constraints can be incorporated using oracle \ac{crb} formulations \cite{ben2010cramer,babadi2008asymptotic}, 
which are based on perfect knowledge of the support set of the parameter vector.
 However, a tailored \ac{crb} that explicitly integrates both sparsity and reparametrization for Laplacian matrix estimation has not been developed, motivating the need for a rigorous theoretical framework that accounts for these structural properties.


In this paper, we develop performance bounds for the estimation of  Laplacian matrices under structural and sparsity constraints.  First, we introduce a reparametrization that maps the constrained estimation problem into a lower-dimensional parameter space, while preserving the structural equality constraints of the Laplacian matrix. Using this reparametrization, 
we derive four versions of the \ac{crb} for complete and sparse graphs, each with different regularity conditions. 
In particular, for sparse graphs we establish two oracle \acp{crb} that are based on the information of the true support set, i.e. the locations of nonzero entries of the Laplacian matrix. 
We analyze the order relations between the four bounds developed in this paper,  and demonstrate the trade-off between bound tightness and the dimensionality of the associated  \ac{fim} that is required to be a non-singular matrix.
Notably, the ordering relation between the two oracle \acp{crb}, established for the first time in this paper, extends beyond Laplacian estimation and applies to general sparse recovery problems.
In the Gaussian setting, we provide a structured matrix form of the \ac{fim} that is free of
expectation operators,  using the Slepian-Bangs formula.

From the practical perspective, we apply the proposed \acp{crb} to three representative applications:  (i) power system topology estimation via admittance matrix recovery, (ii) graph filter identification in diffused \ac{gsp} models, and (iii) precision matrix estimation in \acp{gmrf} under Laplacian constraints.
In numerical simulations, we evaluate the performance of the proposed \acp{crb} and compare them with the \ac{mse} of the \ac{cmle} and the oracle \ac{cmle}. 
 Our results demonstrate that the \ac{cmle} asymptotically achieves the oracle \acp{crb}, and thus, the bounds are informative as benchmarks in these problems. 
 Additionally, the \acp{crb} accurately characterize estimators' performance across different noise and sample-size regimes, and are useful for the identification of the regions where the inequality constraints naturally hold, thus providing valuable insights into the impact of Laplacian constraints on estimation accuracy.

{\em{Notations and Organization:}}
In this paper, vectors and matrices are denoted by boldface lowercase and uppercase letters, respectively.  The $K \times K$ identity matrix is denoted by $\Imat_{K}$, $\onevec$ is a vector of ones, and $\Emat_{i,j}$ is a matrix with a single nonzero entry at position $(i,j)$, which is 1.  The vector $\evec_m$ is the $m$th column of the identity matrix of the relevant order. 
The notations $|\cdot|$, $\otimes$, and $\text{Tr}(\cdot)$ denote the determinant operator, the Kronecker product, and the trace operator, respectively. 
For any positive semi-definite matrices $\Amat$ and $\Bmat$,
$\Amat\succeq\Bmat$ implies that $\Amat-\Bmat$ is a positive semi-definite matrix.  In addition, $\Amat^{-1}$,
$\Amat^\dagger$, and $\|\Amat\|_2$ are the inverse,  Moore-Penrose pseudo inverse,
and Frobenius norm, respectively, of a matrix $\Amat$.
The $\ell_0$ norm, $\|\avec\|_0$, counts the nonzero entries of $\avec$, and its support is denoted by $\text{supp}(\avec)$.  
For a square matrix $\Amat$,  ${\text{Vec}}(\Amat)$  stacks its columns into a vector. 
The operators
 ${\text{Vec}}_{\ell}(\Amat)$ and $\text{Vec}_{u}(\Amat)$, respectively,  stack the 
 columns of the lower and upper triangular elements of $\Amat$, excluding the diagonal, 
 into a column vector, while $\text{Vec}_{d}(\Amat)$ extracts the diagonal elements as a vector. These operations are illustrated in Fig. \ref{Fig_illustration}.
 In addition,
 for any vector $\avec$, ${\text{diag}}(\avec)$ denotes a diagonal matrix where the $(n,n)$th entry is $a_n$, and $\| \avec \|_\Amat^2=\avec^T \Amat \avec$ defines the weighted quadratic norm for a positive semi-definite matrix $\Amat$. 
 Finally, the Jacobian matrix of a vector function $\gvec$  \ac{wrt} the vector $\thetavec$
  is a matrix with the $(m,k)$th element $\big[\frac{{\partial}\gvec(\thetavecsmall)}{{\partial}\thetavecsmall}\big]_{m,k}=\frac{\partial \mathrm{g}_{m}(\thetavecsmall)}{\partial\theta_{k}}$. 

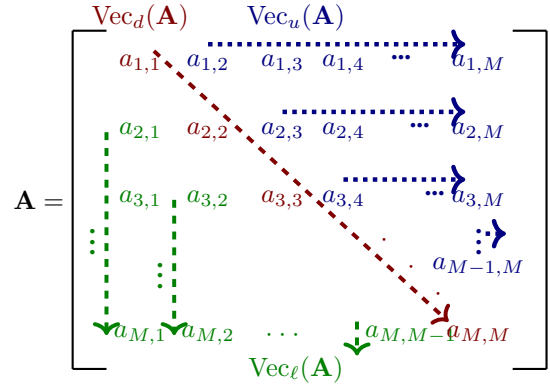
\begin{figure}[hbt]
    \centering
    \begin{tikzpicture}[scale=0.9]
        
        \node at (0,0) {\textcolor{mydarkred}{$a_{1,1}$}};
        \node at (1,0) {\textcolor{mydarkblue}{$a_{1,2}$}};
        \node at (2.1,0) {\textcolor{mydarkblue}{$a_{1,3}$}};
        \node at (3,0) {\textcolor{mydarkblue}{$a_{1,4}$}};
        \node at (5,0) {\textcolor{mydarkblue}{$a_{1,M}$}};

        \node at (0,-1) {\textcolor{mydarkgreen}{$a_{2,1}$}};
        \node at (1,-1) {\textcolor{mydarkred}{$a_{2,2}$}};
        \node at (2.1,-1) {\textcolor{mydarkblue}{$a_{2,3}$}};
        \node at (3,-1) {\textcolor{mydarkblue}{$a_{2,4}$}};
        \node at (5,-1) {\textcolor{mydarkblue}{$a_{2,M}$}};
        
        \node at (0,-2) {\textcolor{mydarkgreen}{$a_{3,1}$}};
        \node at (1,-2) {\textcolor{mydarkgreen}{$a_{3,2}$}};
        \node at (2.1,-2) {\textcolor{mydarkred}{$a_{3,3}$}};
        \node at (3,-2) {\textcolor{mydarkblue}{$a_{3,4}$}};
        \node at (5,-2) {\textcolor{mydarkblue}{$a_{3,M}$}};
        \node at (5,-2.5) {\textcolor{mydarkblue}{\textbf{\vdots}}};

        \node at (5,-3) {\textcolor{mydarkblue}{$a_{M-1,M}$}};
        \node at (0,-4) {\textcolor{mydarkgreen}{$a_{M,1}$}};
        \node at (1,-4) {\textcolor{mydarkgreen}{$a_{M,2}$}};
        \node at (2.1,-4) {\textcolor{mydarkgreen}{\textbf{$\cdot \cdot \cdot$}}}; 
        \node at (4,-4) {\textcolor{mydarkgreen}{$a_{M,M-1}$}};
        \node at (5,-4) {\textcolor{mydarkred}{$a_{M,M}$}};

        \draw[thick] (-1, 0.5) -- (-1, -4.5);
        \draw[thick] (6, 0.5) -- (6, -4.5);
        \draw[thick] (-1, 0.5) -- (-0.5, 0.5);
        \draw[thick] (5.5, 0.5) -- (6, 0.5);
        \draw[thick] (-1, -4.5) -- (-0.5, -4.5);
        \draw[thick] (5.5, -4.5) -- (6, -4.5);
        
        \node at (-1.5, -2) {$\Amat = $};

        \draw [->, line width=1.7pt, dotted, draw=mydarkblue] (1, 0.3) -- (4.8, 0.3) node[near end, below] {\textcolor{mydarkblue}{\textbf{...}}};
        \draw [->, line width=1.7pt, dotted, draw=mydarkblue] (2.1, -0.7) -- (4.8, -0.7) node[near end, below] {\textcolor{mydarkblue}{\textbf{...}}};
        \draw [->, line width=1.7pt, dotted, draw=mydarkblue] (3, -1.7) -- (4.8, -1.7) node[near end, below] {\textcolor{mydarkblue}{\textbf{...}}};
        \draw [->, line width=1.7pt, dotted, draw=mydarkblue] (5.1, -2.5) -- (5.4, -2.5);

        \draw [->, line width=1.5pt, dashed, draw=mydarkgreen] (-0.5, -1) -- (-0.5, -4) node[midway, left] {\textcolor{mydarkgreen}{\textbf{$\vdots$}}};
        \draw [->, line width=1.5pt, dashed, draw=mydarkgreen] (0.5, -2) -- (0.5, -4) node[midway, left] {\textcolor{mydarkgreen}{\textbf{$\vdots$}}};
        \draw [->, line width=1.5pt, dashed, draw=mydarkgreen] (3.2, -3.8) -- (3.2, -4.3);

        \draw [->, line width=1.5pt, dashed, draw=mydarkred] (0.2, 0.2) -- (4.55, -3.8);

        \node at (3.6, -2.6) {\textcolor{mydarkred}{\textbf{$\cdot$}}};
        \node at (4, -3) {\textcolor{mydarkred}{\textbf{$\cdot$}}};
        \node at (4.4, -3.4) {\textcolor{mydarkred}{\textbf{$\cdot$}}};

        \node at (0, 0.7) {\textcolor{mydarkred}{\textbf{$\mathrm{Vec}_{d}(\Amat)$}}};
        \node at (2.3, 0.7) {\textcolor{mydarkblue}{\textbf{$\mathrm{Vec}_{u}(\Amat)$}}};
        \node at (2.3, -4.5) {\textcolor{mydarkgreen}{\textbf{$\mathrm{Vec}_{\ell}(\Amat)$}}};
    \end{tikzpicture}
    \caption{Illustration of the considered operators applied on the matrix $\Amat \in \mathbb{R}^{M\times M}$: \textcolor{mydarkred}{\textbf{$\mathrm{Vec}_{d}(\cdot)$}} returns elements on the main diagonal; \textcolor{mydarkgreen}{\textbf{$\mathrm{Vec}_{\ell}(\cdot)$}} returns elements below the main diagonal; and \textcolor{mydarkblue}{\textbf{$\mathrm{Vec}_{u}(\cdot)$}} returns elements above the main diagonal.}
    \label{Fig_illustration}
\end{figure}

The remainder of this paper is organized as follows. In Section \ref{bounds_sec},
we present the Laplacian matrix
estimation model and develop several versions of the \acp{crb} for both fully connected and sparse graphs. In Section \ref{discussion_section}, we provide theoretical insights into the relationships and properties of the proposed bounds.  In Section \ref{application_section}, we apply
the developed bounds to three representative problems involving Laplacian estimation. 
Simulation results are presented in Section \ref{simulations_section} for the test case of power system topology identification and \ac{lgmrf}. Finally, the
paper is concluded in Section \ref{sec:conclusions}.


\section{CRB for Laplacian Matrix Estimation}
\label{bounds_sec}
In this section, we develop the \ac{crb} for the estimation of the Laplacian matrix, which plays a critical role in many applications (see, e.g. \cite{
Lital2023Smooth,cardoso2020algorithmslearninggraphsfinancial,Shafipour2021Diffusion,sandryhaila2014discrete,shuman2013emerging, halihal2024estimation, Giannakis_Wollenberg_2013, drayer2020detection}). The estimation of the Laplacian matrix is based on incorporating parametric and structural constraints that should be taken into account in the derivation of the \ac{crb}.
In Subsection \ref{Subsection: Laplacian matrix properties}, we present the problem formulation and the Laplacian matrix properties. In addition, we discuss the implications of the parametric constraints on the \ac{crb}. In Subsection \ref{Subsection: Laplacian Reparametrization}, we 
introduce a reparametrization approach to effectively incorporate the constraints. Finally, in Subsections
\ref{crb_subsection} and \ref{section_sparse_crb}, we develop the Laplacian-constrained \ac{crb} for complete graphs and sparse graphs, respectively, using the parametrized model.

\subsection{Problem Formulation and Laplacian Matrix Properties}
\label{Subsection: Laplacian matrix properties}

The Laplacian matrix plays a fundamental role in the field of graph theory and \ac{gsp}, where graphs are used to represent systems of interconnected nodes in many real-world applications.
This matrix captures the structure of the graph by reflecting both the connections between nodes and the edge weights. 
We consider the problem of estimating the Laplacian matrix of an undirected weighted graph  ${\mathcal{G}}$. The graph ${\mathcal{G}}$ consists of a set of nodes (vertices) ${\mathcal{V}}$, a set of edges ${\mathcal{E}}$, 
and a set of positive edge weights $\{W_{e}\}_{e\in{\mathcal{E}}}$. 
We assume that the graph has $M=|{\mathcal{V}}|$ nodes. 
The Laplacian matrix, $\Lmat \in \mathbb{R}^{M \times M}$, is defined such that its $(m,k)$th element is given by:
\begin{equation}
\label{L_def}
L_{m,k} = \begin{cases}
	\displaystyle \sum\nolimits_{m=1}^M  W_{m,k},& m = k \\
  \displaystyle  -W_{m,k},& {\text{otherwise}} 
\end{cases},
\end{equation}
$m,k=1,\ldots,M$. Consequently, $\Lmat $ is a real positive semi-definite matrix. The nonzero elements of $\Lmat$ are either the diagonal elements (i.e. $m=k$) or the off-diagonal elements corresponding to the edges of the graph (i.e. ${\mathcal{E}}_{k,m} \in {\mathcal{E}}$).

Based on this definition, the Laplacian matrix has the following properties:
	\renewcommand{\theenumi}{P.\arabic{enumi}}
\begin{enumerate}
    
    \item\label{P1} Symmetry: $\Lmat=\Lmat^T$
    \item\label{P2} Null-space property: $\Lmat\onevec = \zerovec$
    \item\label{P3} Non-positive off-diagonal entries: $[\Lmat]_{m,k}\leq 0$, \\  $\forall   m,k=1,\dots,M$, $m\neq k$
    \item\label{P4} Positive semi-definiteness: $\Lmat\succeq\zerovec$
    \item\label{P5} Sparsity: $\Lmat$ is a sparse matrix.
\end{enumerate}
Our goal in this paper is to derive the \ac{crb} tailored for the estimation of the Laplacian matrix $\Lmat$, taking into account its structural properties, and to demonstrate its usefulness in commonly-used estimation problems.

\subsubsection{Discussion on the parametric constraints}
In the following, we discuss the relevant parametric constraints for the Laplacian estimation problem based on Properties \ref{P1}-\ref{P5}.
First, we note that Property \ref{P4} is redundant in the presence of Properties \ref{P1}-\ref{P3} (see, e.g. \cite{Morad_ICASSP2022,halihal2024estimation} and p. 392 in \cite{Horn2012}). This is due to the fact that symmetric, diagonally-dominant matrices with nonnegative diagonal entries, such as $\Lmat$ under Properties \ref{P1}-\ref{P3}, are inherently positive semi-definite. Therefore, Property \ref{P4} can be excluded from the CRB development.
Second, the \ac{crb} is a local bound, and as such, inequality constraints do not provide additional information that impacts the bound (see, e.g. \cite{gorman1990lower,Stoica_Ng,Nitzan_constraints}). Consequently, the inequality constraint on non-positive off-diagonal elements of the Laplacian matrix in Property \ref{P3} is excluded from our formulation of the new bound.
Finally, the sparsity constraint in Property \ref{P5} should be taken into account in a similar manner to vector estimation under sparsity constraints \cite{benhaim2010cramer}. As shown in \cite{benhaim2010cramer}, for general sparse recovery the oracle \ac{crb}, which uses the true support set, should be employed. In the considered setting, the oracle \ac{crb} is the \ac{crb} limited to the support of the sparse
Laplacian matrix.
To conclude, the key properties that should be integrated into the \ac{crb} derivations are Properties \ref{P1} and \ref{P2}, while the sparsity property (Property \ref{P5}) should be enforced by aligning the support set with its true value. 

\subsubsection{Which Cram$\acute{\text{e}}$r-Rao-type bound should be used?}
The conventional unconstrained \ac{crb} \cite{Cramer1946,Rao1945,Kayestimation} does not provide a valid bound on the \ac{mse} of Laplacian matrix estimators since it does not take into account the parametric constraints. Thus, constrained versions of the \ac{crb} must be considered.
The recently proposed \ac{lu-ccrb} \cite{Nitzan_constraints} offers a lower bound on the \ac{mse} under linear and nonlinear parametric constraints,  and has been demonstrated to be tighter and more informative than the state-of-the-art \ac{ccrb}. However, the constraints relevant to our analysis,  as discussed in the previous paragraph, are given by Properties \ref{P1} and \ref{P2}, which only imply {\em{linear}} equality constraints on the Laplacian matrix. In this scenario, the \ac{lu-ccrb} \cite{Nitzan_constraints}  coincides with the \ac{ccrb} from \cite{Hero_constraint}. Furthermore, for estimation problems involving {\em{linear}} parametric constraints, reparametrizing the problem and computing the conventional \ac{crb} for the reparametrized vector is equivalent to applying the \ac{ccrb} on the original estimation problem \cite{Moore_scoring,menni2014versatility}. Therefore, the choice between these approaches is a matter of preference. In this paper, we adopt the reparametrization approach, as detailed in the following subsection.

\subsection{Laplacian Reparametrization}
\label{Subsection: Laplacian Reparametrization}
In this subsection, we introduce the reparametrized model of the Laplacian estimation by incorporating the symmetric and null space properties of $\Lmat$ (i.e. Properties \ref{P1} and \ref{P2}). To this end, we first use an identity that relates the different components of any square matrix, i,e. the lower-triangular elements,  the  main diagonal elements,  and the upper-triangular elements, vectorized into $\mathrm{Vec}_{\ell}(\cdot)$, $\mathrm{Vec}_{d}(\cdot)$, and $\mathrm{Vec}_{u}(\cdot)$, respectively, and the vectorized matrix $\mathrm{Vec}(\Lmat)$, as follows \cite{hjørungnes2011complex}:
\begin{equation}
        \label{Connection Between Vectorization Operators}
        \mathrm{Vec}(\Lmat) = \Psimat_{\ell}\mathrm{Vec}_{\ell}(\Lmat) + \Psimat_{d}\mathrm{Vec}_{d}(\Lmat) + \Psimat_{u}\mathrm{Vec}_{u}(\Lmat),
\end{equation}
 \begin{figure*}
        \begin{equation}
                    \begin{aligned}
                        \label{Permutation matrix Ld} 
                     \Psimat_{d} \define \left[\text{Vec}(\evec_{1}\evec_{1}^{T}),\text{Vec}(\evec_{2}\evec_{2}^{T}),\dots,\text{Vec}(\evec_{M}\evec_{M}^{T})\right]\in \mathbb{R}^{M^{2}\times M} .
                    \end{aligned}
                \end{equation}
                \begin{equation}
                    \begin{aligned}
                        \label{Permutation matrix Ll} 
                     \Psimat_{\ell} \define \left[\text{Vec}(\evec_{2}\evec_{1}^{T}),\text{Vec}(\evec_{3}\evec_{1}^{T}),\dots,\text{Vec}(\evec_{M}\evec_{1}^{T}),\text{Vec}(\evec_{3}\evec_{2}^{T}),\dots,\text{Vec}(\evec_{M}\evec_{M-1}^{T})\right]\in \mathbb{R}^{M^{2}\times \frac{M(M-1)}{2}} .
                    \end{aligned}
                \end{equation}
                \begin{equation}
                    \begin{aligned}
                        \label{Permutation matrix Lu} 
                     \Psimat_{u} \define \left[\text{Vec}(\evec_{1}\evec_{2}^{T}),\text{Vec}(\evec_{1}\evec_{3}^{T}),\dots,\text{Vec}(\evec_{1}\evec_{M}^{T}),\text{Vec}(\evec_{2}\evec_{3}^{T}),\dots,\text{Vec}(\evec_{M-1}\evec_{M}^{T})\right]\in \mathbb{R}^{M^{2}\times \frac{M(M-1)}{2}} .
                    \end{aligned}
                \end{equation}
            \end{figure*}
\noindent \hspace{-0.2cm}where $\Psimat_{\ell}$, $\Psimat_{u}\in \mathbb{R}^{M^{2}\times \frac{M(M-1)}{2}}$, and $\Psimat_{d}\in \mathbb{R}^{M^{2}\times M}$ are defined in \eqref{Permutation matrix Ld}-\eqref{Permutation matrix Lu} at the beginning of Page 4. 

 In the following, we show that under the constraints specified by Properties \ref{P1} and \ref{P2}, the entries of the Laplacian matrix can be fully characterized using only the elements from the lower triangular part of $\Lmat$, i.e. through
 \be
\label{alpha_def}\alphavec\define \mathrm{Vec}_{\ell}(\Lmat).
\ee
 First, due to the symmetry constraint (Property \ref{P1}), we have 
 \be 
 \label{parametrized_P1}
 \mathrm{Vec}_{u}(\Lmat)=
 \mathrm{Vec}_{\ell}(\Lmat)=\alphavec.
 \ee 
Second, Property \ref{P2} implies a null space property that allows us to define a linear transformation matrix $\Pmat\in \mathbb{R}^{M\times \frac{M(M-1)}{2}}$ such that 
 \be\label{parametraized_P2}
 \mathrm{Vec}_{d}(\Lmat) = \Pmat \times\mathrm{Vec}_{\ell}(\Lmat)=\Pmat \times \alphavec.
 \ee

The general form of $\Pmat$ is defined such that 
for $i = 1,\dots, M$ and $k = 1, \dots,\frac{M(M-1)}{2}$, its $(i,k)$th elements are
\begin{equation}
\label{matrix P} [\Pmat]_{i,k} = - (\delta_{i, i_k} + \delta_{i, j_k}), \end{equation} 
where $\delta_{a,b}$ is the Kronecker delta function, which equals $1$ if $a = b$ and $0$  otherwise, and $(i_k, j_k)$ are the index pairs corresponding to the $k$th element of $\alphavec = \mathrm{Vec}_{\ell}(\Lmat)$, (that is, 
$\alpha_k=L_{i_k,j_k}$), with $i_k > j_k$. An example illustrating the construction of $\Pmat$ is provided below.

\begin{boxA}[colback=orange!7!white, colframe=brown!80!black, title=\textcolor{black!70!black}{\textbf{Example: Construction of \( \Pmat \)}}]
 mSuppose $M = 4$, so $\frac{M(M-1)}{2} = 6$. According to \eqref{alpha_def}, for this graph the vector $\alphavec$ is given by:
\begin{align*}
    \alphavec = [L_{21}, L_{31}, L_{41}, L_{32}, L_{42}, L_{43}]^T.
\end{align*}
The corresponding index pairs $(i_k, j_k)$ for the elements of $\alphavec$ are
\begin{align*}
    \{(i_k, j_k)\}_{k=1}^{6} \!\!&=&\!\!\!\! \{(2,1), (3,1), (4,1), (3,2), (4,2), (4,3)\}.
\end{align*}
Using these pairs, the matrix $\Pmat$ is constructed as:
\begin{align*}
    \Pmat = 
\begin{bmatrix}
-1 & -1 & -1 &  0 &  0 &  0 \\
-1 &  0 &  0 & -1 & -1 &  0 \\
 0 & -1 &  0 & -1 &  0 & -1 \\
 0 &  0 & -1 &  0 & -1 & -1 \\
\end{bmatrix}.
\end{align*}
Each row $i$ of $\Pmat$ contains $-1$ at positions $k$ where $i = i_k$ or $i = j_k$, ensuring proper mapping from $\alphavec$ to the diagonal elements of $\Lmat$.
\end{boxA}



 By substituting 
\eqref{parametrized_P1} and \eqref{parametraized_P2}
 into the identity in \eqref{Connection Between Vectorization Operators}, we obtain
\begin{equation}
        \label{Laplacian Description using only under elements}
        {\mathrm{Vec}}(\Lmat) = \Psimat\alphavec,
\end{equation}
where   
\be
\label{Psimat_def}
\Psimat \define \Psimat_{\ell} + \Psimat_{d}\Pmat + \Psimat_{u}\in \mathbb{R}^{M^{2}\times \frac{M(M-1)}{2}}.
\ee
 This formulation represents a linear transformation that directly relates $\mathrm{Vec}(\Lmat)$ and $\alphavec=\mathrm{Vec}_{\ell}(\Lmat)$, thereby providing a compact representation of the Laplacian matrix under the given constraints.
Based on the transformation in \eqref{Laplacian Description using only under elements}, we define the new vector for the parameter estimation problem as $\alphavec$
from \eqref{alpha_def}.
Thus, 
in the reparametrized problem we have $\frac{M(M-1)}{2}$ parameters to estimate compared with $M^2$ in the original model that aims to estimate the full parameter vector $\mathrm{Vec}_{\ell}(\Lmat)$. 
Thus, we can write the log-likelihood function of the Laplacian-dependent model as a sole function of $\alphavec$ and achieve a dimensionality reduction.

\subsection{CRB for Complete Graphs}
\label{crb_subsection}

In this and the following subsections, we develop the \acp{crb} for Laplacian matrix estimation. 
The \ac{crb} is a fundamental lower bound on the \ac{mse} of any unbiased estimator of a parameter vector in the non-Bayesian framework \cite{Cramer1946,Rao1945,Kayestimation}.
The proposed \acp{crb} in Subsections \ref{crb_subsection} and \ref{section_sparse_crb} are developed for the problem of estimating the
Laplacian matrix from an observation vector, 
$\xvec$, with a known \ac{pdf}, $f(\xvec;\Lmat)$, parametrized by an unknown, deterministic Laplacian matrix, $\Lmat$.
In this subsection, we derive the \ac{crb} for $\Lmat$ under the symmetry and null space properties, i.e. for the parametrized model with $\alphavec$,  for complete graphs.

Let us define the \ac{fim} based on the entire Laplacian matrix:
\beqna
\label{J_def}
\Jmat_L\big(\mathrm{Vec}(\Lmat)\big)
\define {\mathbb{E}}\left[\left(\frac{\partial \log f(\xvec;\Lmat)}{\partial \mathrm{Vec}(\Lmat)}\right)^T\frac{\partial \log f(\xvec;\Lmat)}{\partial 
 \mathrm{Vec}(\Lmat)}\right].
\eeqna
The matrix $\Jmat_L\big(\Psimat \alphavec\big)$ is obtained by substituting $\mathrm{Vec}(\Lmat)=\Psimat \alphavec$, from \eqref{Laplacian Description using only under elements} in \eqref{J_def}.
The following theorem states the \ac{crb} for the case of a complete graph, i.e. where all the elements of the parameter vector $\alphavec$ are nonzero.  This theorem incorporates linear equality constraints into the estimation problem, yielding a tighter bound than the unconstrained CRB. 
\begin{theorem}[Laplacian-Constrained CRB for Complete Graphs]
\label{Theorem1 - Complete Graph}
Consider the estimation of the vector $\alphavec\in \mathbb{R}^{\frac{M(M-1)}{2}}$, defined in \eqref{alpha_def}, from the observation vector $\xvec$,
assuming a complete graph structure.
  Then, under the regularity conditions of the \ac{crb} (see, e.g.\cite[Chapter~3, pp.~39--45]{Kayestimation}),
the \ac{mse} of any unbiased estimator of $\alphavec$, $\hat{\alphavec}$, satisfies
\begin{equation}
        \label{bound on mse}
        \mathbb{E}[(\hat{\alphavec}-\alphavec)(\hat{\alphavec}-\alphavec)^{T}]\succeq
        \tilde{\Bmat}_i,~i=1,2,
\ee
where
\begin{subequations}
\begin{align}
\label{B1_tilde_def}
\tilde{\Bmat}_1 &\triangleq \Jmat_\alpha^{-1}(\alphavec), \\
\label{B2_tilde_def}
\tilde{\Bmat}_2 &\triangleq \Psimat^{T}\Jmat_L^{-1}\big(\Psimat \alphavec\big) \Psimat,
\end{align}
\end{subequations}
in which
\begin{equation}
        \label{Theorem 2: Fisher Information Matrix}
        \Jmat_\alpha(\alphavec) \define \Psimat^{T}\Jmat_L\big(\Psimat \alphavec\big)\Psimat\in{\mathbb{R}}^{\frac{M(M-1)}{2}\times \frac{M(M-1)}{2}},
\end{equation}
and under the assumption that the \acp{fim} are invertible.
Equality in \eqref{bound on mse} is achieved {\em{iff}} the estimator satisfies
\beqna \label{equality_cond}
 \hat{\alphavec}-\alphavec \!\!\!\!&=&\!\!\!\!\!
 \tilde{\Bmat}_i \times
 \left(
\frac{\partial\log f(\xvec;\alphavec)}{\partial \alphavec}\right)^T,
 \eeqna
for any $\alphavec$ in the parameter space. 
\end{theorem}
\begin{proof}
The \ac{crb} on the estimation error of a general linear transformation
of a deterministic parameter vector can be derived by applying the chain rule to the derivatives of the log-likelihood function (see, e.g. \cite{VanTrees2004,Hai_Messer_Tabrikian}). 
By substituting the specific linear transformation from \eqref{Laplacian Description using only under elements} into the general non-Bayesian \acp{crb}  under linear transformations of the parameter vector provided in Propositions 1 and 2 in \cite{Hai_Messer_Tabrikian}, we obtain the reduced \ac{fim}, $\Jmat_\alpha(\alphavec)$, in \eqref{Theorem 2: Fisher Information Matrix}, as well as the two bounds in \eqref{bound on mse}.
The equality in \eqref{equality_cond} is obtained directly by 
substituting the transformed bounds in the equality condition of the \ac{crb}.
\end{proof}

A discussion on the relation between the two bounds in Theorem \ref{Theorem1 - Complete Graph} appears in Subsection \ref{order_relation_subsection}.  In addition, it can be seen that the equality condition in \eqref{equality_cond} characterizes the conditions under which the CRB is achieved, guiding the design of efficient estimators, e.g. by developing associated Fisher-scoring iterative methods.

\subsection{CRB for Sparse Graphs}
\label{section_sparse_crb}
In many practical applications, such as social networks, biological systems, and sensor networks, the Laplacian matrix is inherently sparse, containing only a few nonzero elements that correspond to the connections between nodes \cite{Yakov2023Laplace,PalomarApproximation,Cardoso2023}. 
This sparsity reflects the limited/nonexistent interactions between entities. While Theorem \ref{Theorem1 - Complete Graph}
provides the \ac{crb} for the estimation of the Laplacian of a complete (fully connected) graph, it does not account for the sparsity typically observed in practical settings.
To address this limitation, in this subsection, we develop a \ac{crb} version that integrates the sparsity structure of the Laplacian matrix, leading to the formulation of the oracle CRB, as detailed in Theorem \ref{Theorem2 - Sparse Graph}.

The following theorem presents the oracle \ac{crb} for the estimation of a sparse Laplacian matrix. The oracle \ac{crb}
is based on the \ac{crb} from Theorem \ref{Theorem1 - Complete Graph}, limited to the true sparsity  support of $\Lmat$. 
To this end, we define the following selection matrix: 
\be
\label{U_def}\Umat_{\alphavecsmall} = [\evec_{i_{1}},\dots,\evec_{i_{s}}]\in{\mathbb{R}}^{\frac{M(M-1)}{2}\times s},
\ee 
where
 $\evec_{i}$ denotes the $i$th standard basis vector in $\mathbb{R}^{\frac{M(M-1)}{2}}$,
$\|\alphavec\|_{0}=s$, and 
${\text{supp}(\alphavec)}=\{i_{1},\dots,i_{s}\}$ is the support of $\alphavec$, which includes the indices of its nonzero elements. Thus, 
$\Umat_{\alphavecsmall}$ is constructed by selecting the columns of the identity matrix $\Imat_{\frac{M(M-1)}{2}}$ corresponding to the indices in $\text{supp}(\alphavec)$, and 
$\Umat_{\alphavecsmall}^T \alphavec$  is a vector that only includes the nonzero entries in $\alphavec$.
\begin{theorem}[Oracle CRB for Sparse Graphs]
\label{Theorem2 - Sparse Graph}
Consider the estimation model described in Theorem \ref{Theorem1 - Complete Graph}, where, in addition, it is known that $\|\alphavec\|_0=s$.  
 Then, under the conventional regularity conditions of the \ac{crb},
the \ac{mse} of  any unbiased estimator of $\alphavec$, $\hat{\alphavec}$, satisfies
\beqna
        \label{bound on mse_sparse}
        {\mathbb{E}}\left[(\Umat_{\alphavecsmall}^T \hat{\alphavec}-\Umat_{\alphavecsmall}^T\alphavec)(\Umat_{\alphavecsmall}^T\hat{\alphavec}-\Umat_{\alphavecsmall}^T\alphavec)^{T}\right]\succeq \Bmat_i,~i=1,2,
\eeqna
where
\begin{subequations}
\begin{align}
\label{B1_def}
\Bmat_1 &\triangleq \left(\Umat_{\alphavecsmall}^{T}\Jmat_\alpha(\alphavec)\Umat_{\alphavecsmall}\right)^{-1}, \\
\label{B2_def}
\Bmat_2 &\triangleq \Umat_{\alphavecsmall}^{T}\Jmat_\alpha^{-1}(\alphavec) \Umat_{\alphavecsmall},
\end{align}
\end{subequations}
and
$\Jmat_\alpha(\alphavec)$ is the \ac{fim} from Theorem \ref{Theorem1 - Complete Graph}.
In addition, 
the oracle \ac{crb}  on the trace of the \ac{mse} of  $\hat{\alphavec}$ is given by
\beqna
        \label{bound on mse_sparse_trace}
  {\mathbb{E}}[(\hat{\alphavec}-\alphavec)^T(\hat{\alphavec}-\alphavec)]\succeq 
       {\text{Tr}} (\Bmat_i),~~~i=1,2.
\eeqna
\end{theorem}
\begin{proof}
    Consider the reparametrized model as considered in Theorem \ref{Theorem1 - Complete Graph}, where the parameter vector is $\alphavec$ and the associated \ac{fim} is $\Jmat_{\alpha}(\alphavec)$. 
    Assume that $\alphavec$ is a sparse vector, $\|\alphavec\|_{0} = s$, with the support set ${\text{supp}}(\alphavec)$ and the selection matrix $\Umat_{\alphavecsmall}$.
    Using $\Umat_{\alphavecsmall}$, the nonzero elements of $\alphavec$ and its estimator $\hat{\alphavec}$ can be expressed as $\Umat_{\alphavecsmall}^T \alphavec$ and $\Umat_{\alphavecsmall}^T \hat{\alphavec}$, respectively. 
    By substituting the new linear transformation,  $\Umat_{\alphavecsmall}^T \alphavec$,
    in the general \acp{crb} under linear transformations provided in Propositions 1 and 2 of \cite{Hai_Messer_Tabrikian}, we obtain the two \ac{crb}-type bounds on the \ac{mse} of the nonzero elements in \eqref{bound on mse_sparse}-\eqref{B2_def}. 

For the total \ac{mse} of $\hat{\alphavec}$, we note that the elements of $\alphavec$ outside its support are zero. Thus, 
\beqna
        \label{mse_sparse}
  {\mathbb{E}}[(\hat{\alphavec}\!\!\!\!\!&-&\!\!\!\!\!\alphavec)^T(\hat{\alphavec}-\alphavec)]\nonumber\\
  \!\!\!\!\!&=&\!\!\!\!\! {\mathbb{E}}\left[(\Umat_{\alphavecsmall}^T \hat{\alphavec}-\Umat_{\alphavecsmall}^T\alphavec)^T(\Umat_{\alphavecsmall}^T\hat{\alphavec}-\Umat_{\alphavecsmall}^T\alphavec)\right]
  \nonumber\\
  &+&\!\!\!\!\! {\mathbb{E}}\left[(\Umat_{/\alphavecsmall}^T \hat{\alphavec}-\Umat_{/\alphavecsmall}^T\alphavec)^T(\Umat_{/\alphavecsmall}^T\hat{\alphavec}-\Umat_{/\alphavecsmall}^T\alphavec)\right]
  \nonumber\\
    &=&\!\!\!\!\! {\mathbb{E}}\left[(\Umat_{\alphavecsmall}^T \hat{\alphavec}-\Umat_{\alphavecsmall}^T\alphavec)^T(\Umat_{\alphavecsmall}^T\hat{\alphavec}-\Umat_{\alphavecsmall}^T\alphavec)\right]
  \nonumber\\
  &+&\!\!\!\!\! {\mathbb{E}}\left[\hat{\alphavec}^T  \Umat_{/\alphavecsmall}\Umat_{/\alphavecsmall}^T\hat{\alphavec}\right]
  \nonumber\\
  &\geq&\!\!\!\!\!  {\mathbb{E}}\left[(\Umat_{\alphavecsmall}^T \hat{\alphavec}-\Umat_{\alphavecsmall}^T\alphavec)^T(\Umat_{\alphavecsmall}^T\hat{\alphavec}-\Umat_{\alphavecsmall}^T\alphavec)\right],
\eeqna
where $\Umat_{/\alphavecsmall}\in{\mathbb{R}}^{\frac{M(M-1)}{2}\times (\frac{M(M-1)}{2}-s)}$ is the complement of $\Umat_{\alphavecsmall}$, which includes all the columns of $\Imat_{\frac{M(M-1)}{2}}$ that are not in  the index set $\text{supp}(\alphavec)$. 
The second equality is obtained by substituting $\Umat_{/\alphavecsmall}^T\alphavec=\zerovec$, since the matrix  $\Umat_{/\alphavecsmall}$ is associated with the indices that are not in the support set of $\alphavec$. The last inequality is obtained by removing a nonnegative term, i.e. using the fact that 
${\mathbb{E}}\left[\hat{\alphavec}^T  \Umat_{/\alphavecsmall}\Umat_{/\alphavecsmall}^T\hat{\alphavec}\right]\geq 0$.
By substituting 
\eqref{bound on mse_sparse} in \eqref{mse_sparse}
and using the fact that $\operatorname{Tr}(\cdot)$ is a linear operator, we obtain the bounds in \eqref{bound on mse_sparse_trace} on the total \ac{mse} trace.
\end{proof}


\section{Discussion and Special Cases}
\label{discussion_section}
In this section, we discuss some properties of the proposed \acp{crb}. 
In Subsection~\ref{bounds_prop} we examine the conditions for the existence of the bounds, investigate their order relation, and discuss their applicability to general sparse recovery problems. In Subsection~\ref{Gaussian_subsec} we address the Gaussian case, and derive the associated bounds using the Slepian-Bangs formula. Finally, in Subsection~\ref{relation_with_cmle} we explore the relationship between the proposed bounds and the \ac{cmle}.

\subsection{Discussion}
\label{bounds_prop} 
\subsubsection{Existence of the Bounds}
\label{order_relation_subsection}
The bound $\tilde{\Bmat}_2$ from \eqref{B2_tilde_def} in Theorem \ref{Theorem1 - Complete Graph} 
requires the invertibility of the  $M^2\times M^2$ \ac{fim}, $\Jmat_L(\mathrm{Vec}(\Lmat))$. 
The bound $\tilde{\Bmat}_1$ from \eqref{B1_tilde_def} in Theorem \ref{Theorem1 - Complete Graph} and the bound $\Bmat_2$ from  \eqref{B2_def} in Theorem \ref{Theorem2 - Sparse Graph} require the invertibility of the $\frac{M(M-1)}{2} \times \frac{M(M-1)}{2}$ \ac{fim}, $\Jmat_\alpha$. 
The bound $\Bmat_1$ from \eqref{B1_def} in Theorem \ref{Theorem2 - Sparse Graph} requires the inversion of a matrix with a significantly smaller dimension, 
$\Umat_{\alphavecsmall}^{T}\Jmat_\alpha(\alphavec)\Umat_{\alphavecsmall}{\in{\mathbb{R}}^{s\times s}}$,  where $s = \|\alphavec\|_0$ denotes the sparsity level of $\alphavec$. 
Since $s \ll \frac{M(M-1)}{2}$ in typical sparse Laplacian estimation problems, the condition for the existence of $\Bmat_1$ is significantly less restrictive than the other bounds. As a result, $\Bmat_1$ can often be computed even in cases where the complete-graph bounds from Theorem \ref{Theorem1 - Complete Graph} and $\Bmat_2$ cannot.
This is thanks to representing a lower-dimension projection of the information onto the subspace defined by a smaller parameter vector.
This property is particularly relevant
in sparse recovery scenarios, where the sparsity assumption renders the estimation problem well-posed in underdetermined settings.  
It can be seen that the invertibility of these bounds
is determined by the rank of the matrix $\Jmat_L(\Psimat \alphavec)$ from \eqref{J_def}, the sparsity level, and the graph dimension $M$.

\subsubsection{Relation with the Oracle CRB from  \cite{ben2010cramer}}
The two bounds in \eqref{bound on mse_sparse_trace} can be interpreted as the oracle \ac{crb} since they use the true support set of $\alphavec$. 
In particular, the bound $\Bmat_2$ is the oracle \ac{crb} from \cite{ben2010cramer} for general sparse vector estimation within the \ac{ccrb} framework \cite{gorman1990lower}. 
In this framework, the selection matrix, $\Umat_\alphavecsmall$, functions as a projection matrix onto the orthogonal subspace of the constrained subspace.

The derivation of $\Bmat_1$ extends the applicability of sparse recovery bounds to more general settings. Unlike $\Bmat_2$, which requires the invertibility of $\Jmat_\alpha(\alphavec)$, $\Bmat_1$ only depends on the reduced \ac{fim} corresponding to the nonzero elements of $\alphavec$. Thus, $\Bmat_1$ remains well-defined even in scenarios where $\Bmat_2$ is not, such as ill-posed problems with insufficient measurements.  
Moreover, $\Bmat_1$ is applicable beyond the Laplacian matrix estimation problem and can be utilized in general sparse recovery problems. 
By introducing $\Bmat_1$, we generalize the oracle \ac{crb} to cases involving singular \ac{fim}, providing a robust framework for addressing underdetermined problems.

\subsubsection{Order Relation}
\label{order_relation_subsec} 
The bounds $\Bmat_{1}$ and $\Bmat_{2}$ correspond to two ways of incorporating the support set of the Laplacian matrix into the \ac{crb}. The bound $\Bmat_{1}$ results from inverting the reduced \ac{fim} corresponding to the nonzero elements of $\alphavec$, while $\Bmat_{2}$ is obtained by projecting the inverse of the full \ac{fim} onto the support of $\alphavec$.
As shown in Proposition 3 of \cite{Hai_Messer_Tabrikian}, if both bounds exist, then $\Bmat_2 \succeq \Bmat_1$.
 Thus, $\Bmat_2$ is a tighter bound and should be preferred whenever it is valid. Similarly, if both bounds exist, then $\tilde{\Bmat}_2 \succeq \tilde{\Bmat}_1$.

The oracle \acp{crb} in Theorem \ref{Theorem2 - Sparse Graph} provide more appropriate bounds in a sparse setting compared to the general \acp{crb} in Theorem \ref{Theorem1 - Complete Graph}.
Thus, in general, $\tilde{\Bmat}_j \succeq \Bmat_i$, $i,j=1,2$. Specifically, estimators that integrate sparsity assumptions may achieve an \ac{mse} below the bounds from Theorem \ref{Theorem1 - Complete Graph}, as these bounds do not account for sparsity. 
On the other hand, the bounds from Theorem \ref{Theorem2 - Sparse Graph} enforce a specific sparsity pattern. If this pattern is mismatched, the bounds from Theorem \ref{Theorem1 - Complete Graph} may outperform it, particularly in asymptotic regimes where estimators naturally conform to the sparsity pattern even without explicit enforcement.
It should also be noted that under mild regularity conditions, all bounds asymptotically converge to the \ac{mse} of estimators with the correct support set.

\subsubsection{Incidence-based Decomposition of Laplacian Matrices}
\label{Alternative Decomposition of the Laplacian matrix}
In addition to the main reparametrization framework presented in \eqref{Laplacian Description using only under elements}, one may leverage the incidence matrix to derive an equivalent linear transformation. Specifically, 
for graphs with $s$ weighted edges, the unweighted incidence matrix $\Bmat\in \mathbb{R}^{M\times s}$ that describes the connectivity of the graph can be defined as follows:
\beqna
\Bmat_{k,m} =
\begin{cases}
1, & {\text{if edge }} e_{k,m}  {\text{ enters vertex }} k \\
-1, & {\text{if edge }} e_{k,m}  {\text{ enters vertex }} m \\
0, & \text{otherwise}
\end{cases},
\eeqna
where the direction of the edges is arbitrary. 

 Using this representation, the Laplacian matrix can be decomposed as
 \beqna
 \label{def_L_with_B}
        \Lmat = \Bmat{\text{diag}}(\Umat_{\alphavecsmall}^T \alphavec)\Bmat^{T}
=\sum_{m=1}^s\Bmat\Emat_{m,m}(\Umat_{\alphavecsmall}^T \alphavec)\evec_m^T\Bmat^{T},
\eeqna
where $\Umat_{\alphavecsmall}^T \alphavec$ contains the weights associated with the graph edges, $\Umat_{\alphavecsmall}$ is  defined in \eqref{U_def}, and $\Emat_{m,m}$ and $\evec_m$ are defined at the end of Section \ref{sec:intro}.
This decomposition separates the weights from the connectivity pattern.
By vectorizing both sides of \eqref{def_L_with_B} and applying properties of Kronecker products, we obtain
\be
    {\mathrm{Vec}}(\Lmat)=\sum_{m=1}^s((\Bmat\evec_m)\otimes (\Bmat\Emat_{m,m})) \Umat_{\alphavecsmall}^T \alphavec.
    \ee
Comparing this result with \eqref{Laplacian Description using only under elements},
the transformation matrix $\Psimat$ can be expressed as 
\beqna
\Psimat=\sum\nolimits_{m=1}^s((\Bmat\evec_m)\otimes (\Bmat\Emat_{m,m})) \Umat_{\alphavecsmall}^T.
\eeqna
Hence, the incidence-based approach provides an alternative derivation of the linear relationship between $\mathrm{Vec}(\Lmat)$ and $\alphavec$.
Similarly, for complete graphs, 
this formulation can be developed based on an $M \times \frac{M(M-1)}{2}$ incidence matrix $\Bmat$.

These transformations can be used for computing the estimators (especially the \ac{cmle}) and the \acp{crb}.
In certain contexts, such as network flow or power systems, where the incidence matrix $\Bmat$ arises naturally, the incidence-based formulation may be more straightforward. On the other hand, the reparametrization approach in \eqref{Laplacian Description using only under elements} directly encodes the Laplacian’s structural constraints and readily extends to other settings (e.g. M-matrices, symmetric matrices, etc.) by adjusting the relevant constraints. It is important to note that, unlike the reparametrization in \eqref{Laplacian Description using only under elements}, this formulation requires knowledge of the Laplacian’s support structure {\em{a priori}}. In practical scenarios, this information may not always be available, which can limit the direct applicability of this decomposition when estimating graphs from data.


\vspace{-0.5cm}
\subsection{The Gaussian Case - Slepian-Bangs Formula}
\label{Gaussian_subsec}
Theorems \ref{Theorem1 - Complete Graph} and \ref{Theorem2 - Sparse Graph} provide bounds applicable to any distribution of observation vectors. In this subsection, we describe the use of the \acp{crb} for Laplacian matrix estimation under the assumption of \ac{iid} Gaussian observations, where both the mean vector and covariance matrix are parametrized by the Laplacian matrix. 
The classic Slepian-Bangs formula \cite{slepian1954estimation,bangs1971array} provides a closed-form element-wise expression of the \ac{fim} in Gaussian models,  which is free of
expectation operators. 
While the element-wise  \ac{fim} for the Gaussian case is well established \cite{Kayestimation}, the matrix form is less commonly addressed. For completeness, 
we first present the matrix form of the Slepian-Bangs formula for estimating a general unknown parameter matrix, $\Thetamat \in \mathbb{R}^{M \times M}$, within a linear-Gaussian model in the following theorem, with further details available in \cite{magnus1979commutation,Magnus1980}.
\begin{theorem}[FIM for General Matrix Estimation]
    \label{Theorem 3 - Slepian-Bagns}
    Consider estimating $\Thetamat$ based on a   Gaussian observation vector, $\xvec\sim \mathcal{N}(\muvec(\Thetamat), \Cmat(\Thetamat))$, where the unknown parameter matrix $\Thetamat$ appears in both the mean vector, $\muvec(\Thetamat)\in \mathbb{R}^{M}$, and the covariance matrix, $\Cmat(\Thetamat)\in \mathbb{R}^{M \times M}$, which is a positive-definite matrix.   Then,  the corresponding $M^{2}\times M^{2}$ \ac{fim} is given by
\beqna
\label{eq: Alternative equation for the bound}
     \Jmat(\mathrm{Vec}(\Thetamat)) = \frac{\partial^T\muvec(\Thetamat)}{\partial\mathrm{Vec}(\Thetamat)}\Cmat^{-1}(\Thetamat)\frac{\partial\muvec(\Thetamat)}{\partial\mathrm{Vec}(\Thetamat)}\hspace{2cm}
        \nonumber \\ +\frac{1}{2}\frac{\partial^{T}\mathrm{Vec}(\Cmat(\Thetamat))}{\partial\mathrm{Vec}(\Thetamat)} (\Cmat^{-1}(\Thetamat)\otimes \Cmat^{-1}(\Thetamat))
        \frac{\partial\mathrm{Vec}(\Cmat(\Thetamat))}{\partial\mathrm{Vec}(\Thetamat)}. 
\eeqna
\end{theorem}
\begin{IEEEproof}
  See in \cite[Chapter~15, pp.~356--357]{magnus2019matrix}.
\end{IEEEproof}


For the special case of $\Thetamat = \Lmat$, the derivatives of the mean and covariance matrices \ac{wrt} the Laplacian parameters given by $\alphavec$ from \eqref{alpha_def} are given by
\begin{equation}
          \label{Theorem 3: Chain Rule Mean}
          \frac{\partial\muvec(\Lmat)}{\partial\alphavec} 
          = \frac{\partial\muvec(\Lmat)}{\partial\mathrm{Vec}(\Lmat)}\frac{\partial\Psimat\alphavec}{\partial\alphavec}= \frac{\partial\muvec(\Lmat)}{\partial\mathrm{Vec}(\Lmat)}\Psimat,
  \end{equation}
  where we used the linear transformation in \eqref{Laplacian Description using only under elements}.
 Similarly, the derivative of the covariance matrix \ac{wrt} $\alphavec$ is 
  \begin{equation}
      \begin{aligned}
          \label{Theorem 2: Chain Rule Covariance}
          &\frac{\partial\mathrm{Vec}(\Cmat(\Lmat))}{\partial\alphavec} = \frac{\partial\mathrm{Vec}(\Cmat(\Lmat))}{\partial\mathrm{Vec}(\Lmat)}\Psimat. 
      \end{aligned}
  \end{equation}
  By substituting \eqref{Theorem 3: Chain Rule Mean}-\eqref{Theorem 2: Chain Rule Covariance} in \eqref{eq: Alternative equation for the bound}, we obtain
  \beqna
  \label{eq: Alternative equation for the bound_G}
        \Jmat_{\alpha}(\alphavec)\!\!\!\!\!&=&\!\!\!\!\!  \frac{\partial^{T}\muvec(\Lmat)}{\partial \alphavec}\Cmat^{-1}(\Lmat)\frac{\partial\muvec(\Lmat)}{\partial\alphavec} \nonumber \\ &+&\!\!\!\!\! \frac{1}{2}\frac{\partial^{T}\mathrm{Vec}(\Cmat(\Lmat))}{\partial\alphavec}(\Cmat^{-1}(\Lmat)\otimes \Cmat^{-1}(\Lmat))
        \frac{\partial\mathrm{Vec}(\Cmat(\Lmat)}{\partial \alphavec} 
        \nonumber \\ &=&\!\!\!\!\!  \Psimat^T\Jmat_L( \Psimat\alphavec)\Psimat,
  \eeqna
where $\Jmat_L(\cdot)$ is obtained by substituting $\Thetamat=\Lmat$ in \eqref{eq: Alternative equation for the bound}.
The result in \eqref{eq: Alternative equation for the bound_G} is the Gaussian version of the \ac{fim} under transformation in \eqref{Theorem 2: Fisher Information Matrix}.
By substituting this \ac{fim} in the \acp{crb} from Theorems \ref{Theorem1 - Complete Graph} and \ref{Theorem2 - Sparse Graph}, we obtain the associated bounds for the Gaussian case.

\vspace{-0.25cm}
\subsection{Relationship with the CMLE}
\label{relation_with_cmle}
The \ac{cmle} maximizes the log-likelihood while incorporating parametric constraints on the unknown parameters to be estimated, and it has appealing asymptotic performance \cite{Moore_scoring,Nitzan_constraints}. In our case,
the parametric constraints are imposed by Properties \ref{P1}-\ref{P5} and are both equality, inequality, and sparsity constraints.
 In this work, the estimation process focuses on the parameter vector $\alphavec$, from which the Laplacian matrix $\Lmat$ is reconstructed using the relationship in \eqref{Laplacian Description using only under elements} that imposes the constraints. 
The estimator of $\alphavec$ can be formulated as
\begin{align}
\label{hat_alpha}
\hat{\alphavec} &= \arg\mspace{-28mu}\max_{\alphavecsmall \in \mathbb{R}^{\frac{M(M-1)}{2}}} \log f(\xvec; \alphavec) \nonumber \\
\text{s.t.} \quad & 
\begin{cases}
    \text{C.1}: \alphavec \leq \zerovec, \\
    \text{C.2}: \alphavec \text{ is sparse}.
\end{cases}
\end{align}
The oracle \ac{cmle}, which assumes perfect knowledge of the true support set of $\alphavec$, is given by
\begin{align}
\hat{\alphavec}^{\text{oracle}} &= \arg\mspace{-28mu}\max_{\alphavecsmall \in \mathbb{R}^{\frac{M(M-1)}{2}}} \log f(\xvec; \Lmat) \nonumber \\
\text{s.t.} \quad & 
\begin{cases}
    \text{C.1}: \alphavec \leq \zerovec, \\
    \text{C.3}: \text{supp}(\alphavec)~ \text{is given}.
\end{cases}
\end{align}
Although the oracle \ac{cmle} is not a practical estimator, its performance serves as a benchmark for evaluating the accuracy of support recovery in sparse recovery problems.

  The proposed \acp{crb} on the Laplacian matrix provides important insights into the behavior and limitations of estimators. 
  Theorem \ref{Theorem1 - Complete Graph} provides the \acp{crb} for complete graphs, where all elements of $\alphavec$ are nonzero.
  These bounds serve as a baseline and represent the estimation performance without leveraging sparsity. 
  Thus, when $\Lmat$ is sparse, the performance of both the \ac{cmle} and the oracle \ac{cmle} may be below this bound, as they inherently incorporate sparsity constraints.  Theorem \ref{Theorem2 - Sparse Graph} introduces the oracle \acp{crb} that leverage the knowledge of the true support set of $\alphavec$ and offer a more appropriate performance benchmark for sparse recovery scenarios. All bounds are local and, thus, do not account for the inequality constraints of Property \ref{P3}, while both the \ac{cmle} and the oracle \ac{cmle} satisfy these inequality constraints since the estimator $\hat{\Lmat}$ is constructed directly from $\alphavec$ using \eqref{Laplacian Description using only under elements}. Thus, in general, in low \ac{snr} regimes, the actual performance of the estimators can be lower than the bounds. The usefulness of the bounds in this context is that they can be used to identify regimes where inequality constraints need not be explicitly enforced. When the inequality constraints are satisfied naturally (without enforcement), the gap between the \ac{cmle} in \eqref{hat_alpha} and the oracle \ac{crb} from Theorem \ref{Theorem2 - Sparse Graph} can be directly attributed to sparsity constraints. This distinction enables the bounds to quantify the impact of sparsity and other constraints on estimation performance. Furthermore, the oracle \ac{crb} also helps to identify asymptotic regimes where the estimators achieve perfect support recovery and satisfy the inequality constraints.

\vspace{-0.25cm}
\section{Applications}
\label{application_section}
In this section, we explore the proposed \acp{crb} from Section \ref{bounds_sec} in several applications where the estimation of the Laplacian matrix plays a critical role. 
In Subsection \ref{power_system_subsec}, we present the \ac{crb} for admittance matrix estimation in power systems. In
Subsection \ref{GSP_model_subsection}, we discuss the new \ac{crb} for the problem of graph filter identification in diffused graph models.
Finally, in Subsection \ref{Subsection: Laplacian Estimation in GMRF}, we investigate the estimation of the Laplacian precision matrix in 
graphical models. 

\vspace{-0.25cm}
\subsection{Topology Estimation in Power System}
\label{power_system_subsec}
In this subsection, we examine the problem of mean vector estimation in a Gaussian linear model, where the Laplacian matrix influences only the expectation (and not the covariance matrix) of the observations. In particular, we discuss a key application of this problem, which is topology identification in power systems, where the network structure is inferred from observations. 
A power system can be represented as an undirected weighted graph, where buses (generators or loads) serve as the nodes, and transmission lines form the edges. In this case, the imaginary component of the admittance matrix, the susceptance matrix, serves as a Laplacian matrix that encodes the network's topology and branch parameters \cite{Grotas_Routtenberg_2019,GridGSP,halihal2024estimation}.
Accurate estimation of the admittance matrix is essential for the monitoring, analysis, and stability assessment of the power system \cite{Giannakis_Wollenberg_2013}. This estimation problem, often referred to as the inverse power flow problem, involves recovering the system's topology and line parameters from power flow measurements.

In the following, we consider the commonly-used \ac{dc} model, which is a linear approximation of the power flow equations widely-used in power system analysis and monitoring \cite{Abur_Gomez_book}.  
Accordingly, the noisy measurements of 
power injections at the $M$ buses over $N$ time samples 
can be modeled as \cite{Abur_Gomez_book,Giannakis_Wollenberg_2013}
\be
\label{DC_model}
\pvec[n]=\Lmat\btheta[n] + \etavec[n],~~~n=0,\dots,N-1,
\ee
where $n$ is the time index,
 $\pvec[n]=[p_1[n],\ldots,p_M[n]]^T$ is the active power vector at time $n$,
 and $\thetavec[n]=[\theta_1[n],\ldots,\theta_M[n]]^T$ is the voltage phase angles vector at time $n$. 
That is, $p_m[n]$ and $\theta_m[n]$ are the real power and voltage angle at the $m$th bus (node) measured at time $n$. Here, The Laplacian matrix, $\Lmat$, is the susceptance matrix (the imaginary part of the admittance matrix, denoted by $\Bmat$ in the power system literature).
The noise sequence, $\{\etavec[n]\}_{n=0}^{N-1}$, is assumed to be an \ac{iid} zero-mean Gaussian noise with a known positive definite covariance matrix, $\Rmat_{\eta}$.
 This setup enables the estimation of the network topology and the identification of the Laplacian matrix from the observation vectors $\{\pvec[n]\}_{n=0}^{N-1}$ (referred to here as $\xvec$), where $\{\thetavec[n]\}_{n=0}^{N-1}$ are assumed to be known.

By
using the vectorization operator and the Kronecker product, the model from \eqref{DC_model} can be reformulated as
\beqna
        \label{DC Model: Alternative form}
         \pvec[n] \!\!\!\!&=&\!\!\!\! 
         (\thetavec^{T}[n]\otimes \Imat_{M})\mathrm{Vec}({\Lmat}) + \etavec[n],
\eeqna
$n=0,\ldots,N-1$.
The \ac{cmle} for the estimation of the Laplacian matrix under this model has been developed in \cite{halihal2024estimation}.

In order to find the single-measurement \ac{fim}, we only need to calculate partial derivatives of the mean and covariance of $\pvec[n]\sim ((\thetavec^{T}[n]\otimes \Imat_{M})\mathrm{Vec}({\Lmat}),\Rmat_{\eta})$. It can be seen that
    \be
            \label{DC Model: Derivative of the Mean}
            \frac{\partial(\thetavec^{T}[n]\otimes \Imat_{M})\mathrm{Vec}({\Lmat})}{\partial\mathrm{Vec}(\Lmat)} = (\thetavec^{T}[n]\otimes \Imat_{M})
            \ee
            and
            \be
            \label{DC Model: Derivative of the Covariance}
\frac{\partial\mathrm{Vec}(\Rmat_{\eta})}{\partial\mathrm{Vec}(\Lmat)} = \Zeromat,
    \ee
since, in this case, the Laplacian matrix only parametrizes the mean vector. By substituting \eqref{DC Model: Derivative of the Mean} and  \eqref{DC Model: Derivative of the Covariance} in \eqref{eq: Alternative equation for the bound_G}, we obtain the following  expression for the single-measurement \ac{fim}:
\begin{equation}
\label{fim_dc}
        \Jmat_{n}(\alphavec)=\Psimat^{T}(\thetavec[n]\otimes\Imat_{M})\Rmat^{-1}_{\eta}(\thetavec^{T}[n]\otimes \Imat_{M})\Psimat.
\end{equation}

It is well known that 
the total \ac{fim} for $N$ independent samples is the sum of the per-sample \acp{fim}, and is given by (see \cite[Chapter~3, pp.~33--35]{Kayestimation})
\be
\label{total_fim}
\Jmat
(\alphavec
)=\sum_{n=0}^{N-1}\Jmat_{n}(\alphavec).
\ee 
By substituting \eqref{fim_dc} in
 \eqref{total_fim}, the $N$-samples \ac{fim} is given by
\be
\label{total_fim2}
\Jmat
(\alphavec)=\Psimat^{T}\sum_{n=0}^{N-1}
(\thetavec[n]\otimes\Imat_{M})\Rmat^{-1}_{\eta}(\thetavec^{T}[n]\otimes \Imat_{M})\Psimat.
\ee


In practical scenarios, measurements in power systems might be incomplete due to sensor failures or communication issues. To account for this, we modify the model in \eqref{DC_model} by introducing diagonal zero-one selection matrices $\Smat[n]\in{\mathbb{R}}^{M\times M}$ that indicate the available measurements at time $n$:
 \begin{equation} \label{missing_measurements_model} 
\Smat [n]\pvec[n] = \Smat[n] \Lmat \thetavec[n] + \Smat[n] \etavec[n], \quad n = 0, \dots, N-1.\end{equation} 
Under this model, the noise term, $\Smat[n] \etavec[n]$, has zero mean and a covariance matrix given by $\Smat[n]\Rmat_{\eta} \Smat^{T}[n]$.
Similar to the derivation of  \eqref{total_fim2}, the $N$-samples \ac{fim}  under a missing measurements model \eqref{missing_measurements_model}  
is given by
\beqna\label{modified_FIM}
 \Jmat_{\text{missing}}(\alphavec) = \Psimat^{T} \sum_{n=0}^{N-1}\!\!\!\!\!\!\!&(&\!\!\!\!\!\!\!\thetavec[n]\otimes \Smat[n]) 
 (\Smat[n]\Rmat_{\eta} \Smat^{T}[n])^{-1} \nonumber \\ &\times&\!\!\!\!\!\! (\thetavec^{T}[n]\otimes \Smat[n]) \Psimat. 
 \eeqna
 By substituting this \ac{fim} in the \acp{crb} from Theorems \ref{Theorem1 - Complete Graph} and \ref{Theorem2 - Sparse Graph}, we obtain the associated bounds for the problem of topology identification in power systems with missing measurements.

The derived \ac{fim}, $\Jmat_{\text{missing}}(\alphavec)$, enables the computation of \ac{crb}s for various scenarios with incomplete data, and can provide critical insights into the identifiability of power system topologies. 
The identifiability of the system is defined as the ability to uniquely determine the power system topology represented by $\Lmat$, given the available data, represented here by $\{\Smat [n]\pvec[n]\}_{n=0}^{N-1}$. 
A sufficient condition for 
 identifiability under incomplete observations is that the matrix $ \Jmat_{\text{missing}}(\alphavec) $ will be a non-singular matrix. The sparsity of $\Lmat$ plays a crucial role in this context.  
These bounds can also guide the design of sensor placement strategies to ensure identifiability and optimize the network monitoring capabilities \cite{dabush2025efficientsamplingallocationstrategies}.

\subsection{Graph Filter Identification in Diffused Model}
\label{GSP_model_subsection}
In this subsection, we consider the problem of identifying graph filters within a diffused model, where the input signal $\zvec$ has a zero-mean Gaussian distribution with a known covariance matrix, $\Cmat_{z}$.
This setup describes the inference of non-stationary graph signals diffused over an unknown network structure, as discussed, e.g., in Section II in \cite{Shafipour2021Diffusion}, for a general distribution and general \ac{gso}. 
 We focus on the case where the \ac{gso} is a Laplacian matrix, $\Lmat$.

Specifically, suppose the output graph signal $\xvec$ is generated by  
\begin{equation}
        \label{Graph Filter Identificatin: Diffusion Generic Equation} 
         \xvec = \Hmat\zvec = \sum_{f=0}^{F-1}h_{f}\Lmat^{f}\zvec,
\end{equation}
where  $\Hmat\in{\mathbb{R}}^{M\times M}$ is a graph filter represented as a polynomial in $\Lmat$, and $\zvec\sim\mathcal{N}(\zerovec,\Cmat_{z})$. We assume for simplicity that the vector of coefficients $\hvec=[h_0,h_1,\ldots,h_{F-1}]^T$ is known. The goal is to recover the sparse Laplacian, which encodes direct
relationships between the elements of $\xvec$ from observable indirect
relationships generated by a diffusion process.
The covariance matrix of the output signal in the general non-stationary setting is given by
\beqna
\label{Diffusion Model: Covariance matrix of output}
\Cmat_{x} = \Hmat\Cmat_{z}\Hmat.
\eeqna
The non-stationary case of $\Cmat_{z} \neq \Imat_{M}$ implies that the covariance $\Cmat_{x}$ does not necessarily share eigenvectors with $\Lmat$. Nevertheless, the eigenvectors of $\Lmat$ are preserved in the graph filter $\Hmat$, as $\Hmat$ is a polynomial of $\Lmat$. Thus, estimating $\Hmat$ allows the estimation of $\Lmat$.

To quantify how accurately the Laplacian can be inferred from the available observations, we derive the
\ac{fim}.  
To derive the \ac{fim}, we need the partial derivatives of both the mean and covariance matrix of $\xvec$ \ac{wrt} $\mathrm{Vec}_{\ell}(\Lmat)$. Since $\xvec$ has a zero-mean distribution, its derivative \ac{wrt} $\mathrm{Vec}_{\ell}(\Lmat)$ is zero.
With regard to the derivative of the covariance matrix from \eqref{Diffusion Model: Covariance matrix of output}, we obtain by using matrix calculus results (see, e.g.\cite[Chapter~9, pp.~205--207]{magnus2019matrix})  and  $\Hmat = \sum_{f=0}^{F-1}h_{f}\Lmat^{f}$ that
\beqna
\label{Diffusion Model: Derivative of the Covariance Part 1}
   \frac{\partial \mathrm{Vec}(\Cmat(\Lmat))}{\partial\mathrm{Vec}_{\ell}(\Lmat)} = \sum_{f=0}^{F-1}\sum_{j=1}^{f}h_{f} \hspace{4cm} \nonumber \\ \times  \bigg( \!\! \Big(\Hmat\Cmat_{z}\Lmat^{f-j}\otimes\Lmat^{j-1} \Big) + \Big(\Lmat^{f-j}\otimes \Hmat\Cmat_{z}\Lmat^{j-1}\Big) \!\! \bigg)\Psimat. 
\eeqna
Thus, the resulting \ac{fim} for the parameter $\alphavec$ is
 \beqna
\label{Diffusion Model: The FIM for the lower}
\Jmat(\alpha) =  \Psimat^{T} 
\sum_{f_{1}=0}^{F-1}\sum_{f_{2}=0}^{F-1}\sum_{i=1}^{f_{1}}\sum_{j=1}^{f_{2}}h_{f_{1}}h_{f_{2}} \hspace{2.5cm}\nonumber \\ \times\bigg((\Lmat^{f_{1}-i}\Cmat_{z}\Hmat\otimes \Lmat^{i-1}) + (\Lmat^{f_{1}-i}\otimes \Lmat^{i-1}\Cmat_{z}\Hmat)\bigg) \hspace{0.3cm}\nonumber \\ \times
\bigg((\Hmat\Cmat_{z}\Hmat)^{-1}\otimes (\Hmat\Cmat_{z}\Hmat)^{-1}\bigg) \hspace{3.1cm}\nonumber \\ \times
\bigg(\hspace{-0.1cm}(\Hmat\Cmat_{z}\Lmat^{f_{2}-j}\otimes \Lmat^{j-1}) \hspace{-0.05cm}+ \hspace{-0.05cm}(\Lmat^{f_{2}-j}\hspace{-0.05cm}\otimes\Lmat^{j-1}\Cmat_{z}\Hmat)\hspace{-0.1cm}\bigg)\Psimat. 
\eeqna

\subsubsection{Identifiability through Stationarity}
\label{Identifiability through Stationarity}
When the input covariance satisfies $\Cmat_{z} = \Imat_{M}$, the \ac{fim} in \eqref{Diffusion Model: The FIM for the lower} is reduced  to
 \beqna
\label{Diffusion Model: The FIM for stationary case step 1}
\!\!\!\!\!\!&\Jmat&\!\!\!\!\!\!(\alpha) =  \Psimat^{T} 
\sum_{f_{1}=0}^{F-1}\sum_{f_{2}=0}^{F-1}\sum_{i=1}^{f_{1}}\sum_{j=1}^{f_{2}}h_{f_{1}}h_{f_{2}} \bigg(\Lmat^{f_{1}-i}\otimes \Lmat^{i-1}\bigg) \nonumber \\ &\times&\!\!\!\!\!\!
\bigg(\Imat_{M}\otimes \Hmat^{-1} + \Hmat^{-1}\otimes \Imat_{M}\bigg)^{2}   
\bigg(\Lmat^{f_{2}-j}\otimes \Lmat^{j-1}\bigg)\Psimat,
\eeqna
where we use the fact that, since $\Hmat$ is a polynomial in $\Lmat$, we have $\Lmat^{f}\Hmat = \Hmat\Lmat^{f}$, $f=1,\ldots,F$.
Moreover, by using the spectral decomposition $\Lmat = \Vmat\Lambdamat\Vmat^{T}$, we get the following:
\begin{equation}
\begin{aligned}
\label{Diffusion Model: The FIM for stationary case step 2}
&\Jmat(\alpha) =  \Psimat^{T} (\Vmat\otimes \Vmat)
\sum_{f_{1}=0}^{F-1}\sum_{f_{2}=0}^{F-1}\sum_{i=1}^{f_{1}}\sum_{j=1}^{f_{2}}h_{f_{1}}h_{f_{2}} \\&\quad \bigg(\Lambdamat^{f_{1}-i}\otimes \Lambdamat^{i-1}\bigg)
\bigg(\Imat_{M}\otimes \Lambdamat_{H}^{-1} + \Lambdamat_{H}^{-1}\otimes \Imat_{M}\bigg)^{2} \\&\quad 
\bigg(\Lambdamat^{f_{2}-j}\otimes \Lambdamat^{j-1}\bigg)(\Vmat\otimes \Vmat)^{T}\Psimat, 
\end{aligned}
\end{equation}
 where $\Lambdamat_{h} \define \sum_{f=0}^{F-1}h_{f}\Lambdamat^{f}$.

From the result in \eqref{Diffusion Model: The FIM for stationary case step 2}, we can see that for $\Cmat_{z} = \Imat_{M}$, the \ac{fim} is diagonalizable by the Laplacian eigenvector matrix, $\Vmat$. 
This reflects the fact that
 the estimation problem in this case is simplified to the estimation of eigenvalues from data, as $\Vmat$ can be obtained from the output covariance eigenvectors.
This diagonalized form of the \ac{fim} provides insights into the robustness of recovering eigenvectors and eigenvalues. If the eigenvectors are slightly perturbed due to noise or model mismatch, their impact on eigenvalue estimation can be analyzed through the \ac{fim}. When the eigenvalues are closely spaced or degenerate, the \ac{fim} becomes nearly singular, making estimation difficult. In contrast, well-separated eigenvalues improve the conditioning of the \ac{fim}, leading to the more stable and accurate inference of the graph topology.

\subsubsection{Input Covariance Rank}
\label{Role of Input Covariance Rank}
If the covariance matrix $\Cmat_{z}$ is not full-rank, certain eigenvalues of $\Lmat$ become unidentifiable, as the observed covariance, $\Cmat_{x}$, does not certain sufficient information to estimate them. Additionally, in the non-stationary case where $\Cmat_{x}$ and $\Lmat$ are not jointly diagonalizable, the ability to recover the eigenvectors of $\Lmat$ may also be compromised. This effect is manifested in the structure of the \ac{fim} in \eqref{Diffusion Model: The FIM for the lower}, where a low-rank $\Cmat_{z}$ reduces the effective rank of $\Cmat_{x}$. Thus, the rank deficiency of the \ac{fim}, resulting from the low-rank structure of $\Cmat_{x}$, limits the identifiability of the graph topology.

\subsubsection{Influence of Filter Order $F$}
\label{Influence of Filter Order}
Higher-order filters encode multi-hop interactions that capture more global structural information in the diffused observations. As shown by the \ac{fim} expression in \eqref{Diffusion Model: The FIM for the lower}, increasing the filter order $F$ generally improves the model identifiability by increasing the rank and conditioning of the information matrix. The summation in \eqref{Diffusion Model: The FIM for the lower} consists of multiple \ac{psd} matrices, and the rank of a sum of \ac{psd} matrices can only increase (see e.g. Observation $7.1.3$ in \cite{Horn2012}). However, by the Cayley-Hamilton theorem, any polynomial function of the Laplacian beyond degree $M-1$  becomes redundant,  limiting the benefit of increasing $F$ beyond this threshold. Consequently, while a larger $F$ initially improves identifiability and estimation accuracy, no further gain is achieved for $F \geq M$. 
In summary, higher-order diffusion improves estimation performance, 
but only up to a filter order of $F = M-1$.

\subsection{Laplacian Estimation in GMRF}
\label{Subsection: Laplacian Estimation in GMRF}
 \acp{gmrf} provide a widely used statistical framework for learning sparse dependency structures in multivariate data.
The \ac{lgmrf} is based on adding Laplacian structural constraints, which makes this framework particularly relevant in applications where the underlying signals exhibit smooth variations over a network, such as those arising in power systems, social networks, and biological data \cite{shuman2013emerging,Egilmez_Pavez_Ortega_2017,Lital2023Smooth}.
Under the \ac{lgmrf} model, the problem of graph learning can be formulated as a \ac{mle} of the precision matrix under the assumption that it satisfies Laplacian structural constraints.
Additionally, a sparsity-promoting penalty can be incorporated to encourage parsimony in the inferred graph structure.

In this case, we consider a set of independent data samples, $\{\xvec[n]\}_{n=0}^{N-1}$, drawn from a zero-mean Gaussian distribution parametrized by a precision matrix $\Lmat$, such that 
\be 
\label{lgmrf_eq} \xvec[n] \sim \mathcal{N}(\zerovec, \Lmat^\dagger),
\ee
where $\Lmat$ is a Laplacian matrix satisfying Properties \ref{P1}-\ref{P5}.

 Given the observation vector $\xvec=[\xvec^T[0],\ldots\xvec^T[N-1]]^T $ drawn from an \ac{lgmrf}, its log-likelihood is given by 
\beqna
\label{GMRF_log_likelihood} \log f(\xvec;\Lmat)=  {\text{const}} -  {\text{Tr}}(\Lmat\Smat) -\log(|\Lmat|_+),
\eeqna
where $\Smat\define \frac{1}{N}\sum_{n=0}^{N-1}\xvec[n]\xvec^{T}[n]$ is the empirical covariance matrix, ${\text{const}}$ is an additive constant term independent of $\Lmat$, and $|\Lmat|_{+}$ denotes the pseudo-determinant of $\Lmat$ (product of its nonzero eigenvalues).   The objective function in \eqref{GMRF_log_likelihood} can be regularized to promote sparsity in the precision matrix by adding a penalty term, as in \cite{Egilmez_Pavez_Ortega_2017,Yakov2023Laplace}.

The objective function in \eqref{GMRF_log_likelihood} involves the pseudo-determinant term, since the matrix $\Lmat$ is a singular matrix. In particular, we assume a connected graph, and thus,  the Laplacian matrix $\Lmat$  has a rank of $M-1$. The smallest eigenvalue of $\Lmat$ is zero, and it is associated with the eigenvector $\frac{1}{\sqrt{M}}\onevec$. Thus,  in order to avoid the use of the singular matrix $\Lmat$, we use the relation \cite{Yakov2023Laplace,PalomarApproximation}:
\be
\label{relation_L}
|\Lmat|_+= |\Lmat+\Dmat|,
\ee 
where  $\Dmat=\frac{1}{M} \onevec \onevec^T$. By substituting \eqref{relation_L} in \eqref{GMRF_log_likelihood}, the log-likelihood can be written as
\begin{equation}
    \begin{aligned}
        \label{Alternative form}
        \log f(\xvec;\Lmat)=  {\text{const}} -\frac{N}{2}\text{Tr}(\Lmat \Smat )  -\frac{N}{2}\log(|\Lmat+\Dmat|).
    \end{aligned}
\end{equation}

The gradient of the log-likelihood \eqref{Alternative form} \ac{wrt} $\Lmat$ is given by (see, e.g.\cite[Chapter~8, pp.~167--169]{magnus2019matrix})
\beqna
        \label{grad2}
        \frac{\partial \log f(\xvec;\Lmat)}{\partial \Lmat}=   -\frac{N}{2}\Smat   -\frac{N}{2}(\Lmat+\Dmat)^{-1}
       .
\eeqna
The \ac{fim} from \eqref{J_def} can also be computed via
\beqna
\label{J_def2}
\Jmat_L\big(\mathrm{Vec}(\Lmat)\big)
\define -  {\mathbb{E}}\left[\frac{\partial }{\partial \mathrm{Vec}(\Lmat)}\left(\frac{\partial \log f(\xvec;\Lmat)}{\partial 
 \mathrm{Vec}(\Lmat)}\right)^T\right].
\eeqna
Based on \eqref{grad2}, it can be seen that the Hessian of the log-likelihood does not depend on the data $\xvec$, as in many Gaussian settings, implying that the expectation is trivial.  
By substituting \eqref{grad2} in \eqref{J_def2}, we obtain that the \ac{fim} in this case is
\beqna
\label{JL_GMRF}
\Jmat_L\big(\mathrm{Vec}(\Lmat)\big) = 
\frac{N}{2} \left((\Lmat + \Dmat)^{-1} \otimes (\Lmat + \Dmat)^{-1}\right).
\eeqna 
Equivalently, according to \eqref{Theorem 2: Fisher Information Matrix}, the \ac{fim} for $\alphavec$ is 
\begin{equation}
    \begin{aligned}
        \label{GMRF: The Fisher Information Matrix}
        &\Jmat(\alphavec) = \frac{N}{2}\Psimat^{T}
        \bigg((\Lmat+\Dmat)^{-1}\otimes(\Lmat+\Dmat)^{-1}\bigg)\Psimat.
    \end{aligned}
\end{equation}

{\em{Discussion:}} First, we note that since the rank of a Kronecker product $\Amat  \otimes \Bmat$ is ${\text{rank}}(\Amat){\text{rank}}(\Bmat)$, then the rank of $\Jmat_L$ from \eqref{JL_GMRF} is $M^2$. Thus, the full \ac{fim} $\Jmat_L$ is a full-rank matrix in this case. Thus, in this case, all the bounds from Theorems \ref{Theorem1 - Complete Graph} and \ref{Theorem2 - Sparse Graph}
 can be computed.  
Another insight is that since the precision matrix $\Lmat$ is the inverse of the covariance matrix of $\xvec$, we observe that stronger graph connectivity (i.e. larger $\Lmat$ entries) leads to smaller covariance values, improving estimation accuracy. Intuitively, a highly connected graph leads to less variability in observations, tightening the CRB. Thus, If a node has weak connectivity (i.e. small Laplacian entries in its row/column), then its impact on the inverse covariance $\Cmat_x$ is small. Consequently, weakly connected nodes are harder to estimate because they contribute less to the overall likelihood. 

A special case has arisen for block-diagonal patterns. 
If the graph consists of $K$ disconnected components, the Laplacian matrix $\Lmat$ can be written in block-diagonal form:
\begin{equation}
    \Lmat =
    \begin{bmatrix}
    \Lmat_1 & 0 & \cdots & 0 \\
    0 & \Lmat_2 & \cdots & 0 \\
    \vdots & \vdots & \ddots & \vdots \\
    0 & 0 & \cdots & \Lmat_K
    \end{bmatrix},
\end{equation}
where each $\Lmat_k \in \mathbb{R}^{M_k \times M_k}$ represents the Laplacian of the $k$-th connected component, with $M_k$ denoting the number of nodes in that component. Since each $\Lmat_k$ has a single zero eigenvalue, the overall Laplacian $\Lmat$ has rank $M - K$ instead of $M-1$.
To ensure full-rank correction, we define the following matrix:
\begin{equation}
    \bar{\Dmat} = \sum_{k=1}^{K} \frac{1}{M_k} \onevec_k \onevec_k^T,
\end{equation}
where $\onevec_{k} = \sum_{i\in \mathcal{V}_{k}}\evec_{i}\in \{0,1\}^{M}$ is a vector of ones associated with the nodes in the $k$-th connected component, i.e. $\mathcal{V}_{k}$. This generalizes the standard rank-one correction $\Dmat$ used for fully connected graphs.

Using this correction, in a similar way to the derivation of \eqref{JL_GMRF}, the \ac{fim} can be written in this case as
\begin{equation}
    \Jmat_L\big(\mathrm{Vec}(\Lmat)\big) = \frac{N}{2} 
    \bigg( (\Lmat+\bar{\Dmat})^{-1} \otimes (\Lmat+\bar{\Dmat})^{-1} \bigg) .
\end{equation}
Since $\Lmat+\bar{\Dmat}$ remains a block-diagonal matrix, its inverse also follows a block structure:
\beqna
    (\Lmat+\bar{\Dmat})^{-1} =\hspace{5.5cm}
    \nonumber\\
    \begin{bmatrix}
    (\Lmat_1 + \Dmat_1)^{-1} & 0 & \cdots & 0 \\
    0 & (\Lmat_2 + \Dmat_2)^{-1} & \cdots & 0 \\
    \vdots & \vdots & \ddots & \vdots \\
    0 & 0 & \cdots & (\Lmat_K + \Dmat_K)^{-1}
    \end{bmatrix}.
\eeqna
Taking the Kronecker product, the full FIM $\Jmat_L\big(\mathrm{Vec}(\Lmat)\big)$ can be expressed as a block diagonal matrix 
where each block corresponds to a separate subgraph, given by
\begin{equation}
    \Jmat_k(\mathrm{Vec}(\Lmat))= \frac{N}{2}
    \bigg( (\Lmat_k+\Dmat_k)^{-1} \otimes (\Lmat_k+\Dmat_k)^{-1} \bigg),
\end{equation}
$k=1,\ldots,K$.
This block structure of the \ac{fim} implies that each connected component can be estimated independently without interference from other subgraphs.
    That is, no information is shared between different subgraphs, meaning that poor connectivity in one component does not affect others.


 \section{Simulations}
 \label{simulations_section}
In this section, we evaluate the performance of the proposed \acp{crb} for Laplacian matrix estimation in two key applications: power system topology estimation (Subsection \ref{simulations_power_subsection}) and graphical models under the \ac{lgmrf} framework (Subsection \ref{simulation_GMRF_sec}).  We compare the performance of different estimation methods with the \acp{crb} developed in Section \ref{bounds_sec}.
We focus on the \acp{crb} from Theorem \ref{Theorem2 - Sparse Graph}, as the considered applications involve sparse graphs. The bounds from Theorem \ref{Theorem1 - Complete Graph} serve as an intermediate step in deriving the sparse bounds, and can also be used as upper bounds on the \ac{mse} of sparse estimators.

\subsection{Topology Estimation in Power Systems}
\label{simulations_power_subsection}
Accurate estimation of the admittance matrix is essential for monitoring, analyzing, and securing power networks. 
In this subsection, we evaluate the performance of the  \acp{crb} 
 for the \ac{dc} model described in Subsection \ref{power_system_subsec}, and compare the results with the \ac{mse} obtained by 
 the \ac{cmle} and the oracle \ac{cmle}, as developed in \cite{halihal2024estimation}.

The simulations are conducted on the IEEE 33-bus system, with parameters from \cite{iEEEdata}.
In Figure \ref{fig: Admittance matrix illustration} we present the associated susceptance matrix, $\Bmat$, used in our simulation.
All results are averaged over $200$ Monte-Carlo simulations, using a range of $35$ different \ac{snr} levels.
\begin{figure}[hbt]
    \centering
\includegraphics[width=0.85\linewidth]{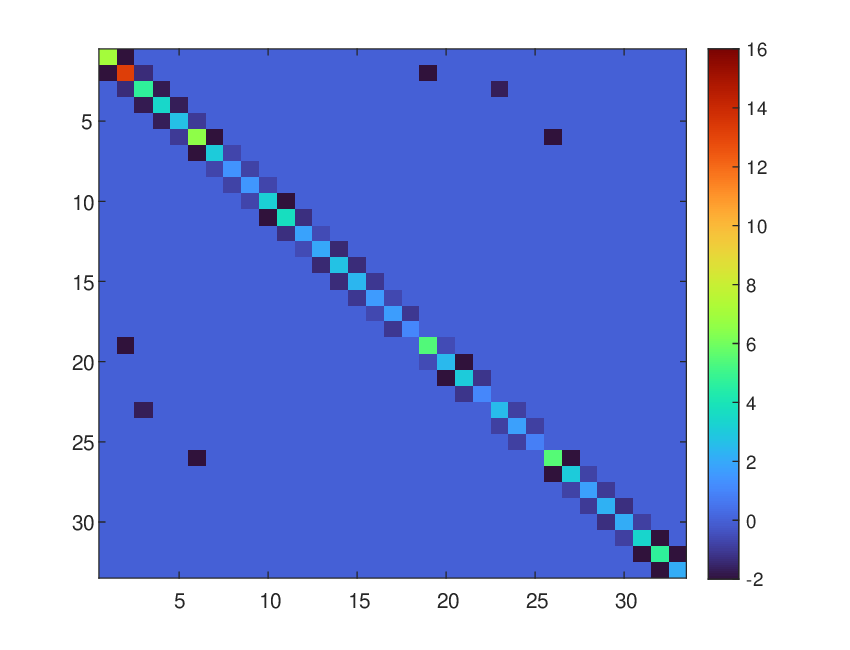}
    \caption{The susceptance matrix, $\Lmat$, for IEEE 33-bus system.}
    \label{fig: Admittance matrix illustration}
\end{figure}

For the \ac{cmle} and the oracle \ac{cmle}, we applied the \ac{admm} framework with a regularization parameter $\lambda = 5$ and a penalty parameter of $\rho = 0.001$. These parameters were fine-tuned experimentally.
Figure \ref{fig: Power System Figure 1} presents the \ac{mse} of the estimators of ${\Lmat}$, along with the proposed oracle \ac{crb}, as a function of the \ac{snr} for $N=600$, where $N$ is the number of available measurements. The \ac{snr} is defined as 
\begin{equation}
\text{SNR} = 10\log\left(\frac{1}{MN \sigma^{2}} \sum_{n=0}^{N-1} \left\| \pvec[n] \right\|_{2}^2 \right).
\end{equation}

As shown in the figure, the oracle \ac{crb} serves as a valid bound for the oracle \ac{cmle}. In addition, it can be seen that both the \ac{cmle} and the oracle \ac{cmle} asymptotically achieve the oracle \ac{crb}. 
Moreover, increasing the number of measurements improves support identification, bringing the \ac{cmle} even closer to the oracle \ac{crb}, particularly at higher \ac{snr} values. For larger numbers of measurements, the \ac{cmle} and oracle \ac{cmle} more closely coincide. This suggests that a higher number of observations leads to better support identification, reducing the gap between the estimators and the bound. 
Finally, it can be seen that the bound $\Bmat_{2}$ is consistently tighter than $\Bmat_{1}$  across all \ac{snr} levels. This validates the theoretical ordering of the bounds in Subsection \ref{order_relation_subsec}.

\begin{figure}[hbt]
    \centering
\includegraphics[width=0.8\linewidth]{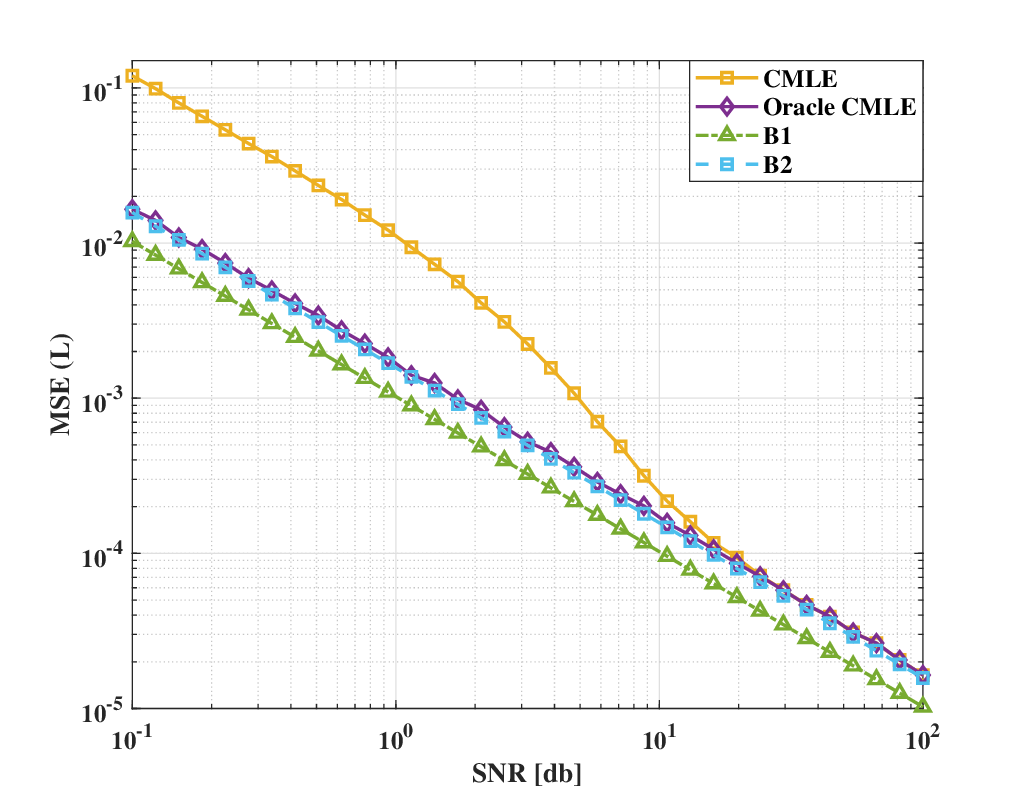}
    \caption{The \ac{mse} of the different estimators compared with the oracle \ac{crb} for $\Lmat$, for IEEE $33$-bus system with $N=600$.}
    \label{fig: Power System Figure 1}
\end{figure}


\subsection{Laplacian Estimation in GMRF}
\label{simulation_GMRF_sec}
For the \ac{lgmrf} model, the evaluation of the estimators and bounds is conducted using the random planar graph model as the underlying topology.
In particular, 
we performed experiments on a random planar graph with $100$ nodes (see Fig. \ref{fig: Laplacian Matrix GMRF 100 Nodes} for an illustration), following the methodology outlined in \cite{Yakov2023Laplace}. The graph's adjacency matrix was constructed with edge weights sampled uniformly from $[0.5,2]$. The corresponding Laplacian matrix was then used to generate zero-mean Gaussian data with pseudo-inverse precision matrix $\Lmat^{\dagger}$, as described in \eqref{lgmrf_eq}.

The simulations were averaged over $150$ Monte Carlo realization, 
and the sample size ratio $\frac{n}{p}$ was used as a key parameter. 
The performance is evaluated using the \ac{re} as a metric, which measures the Frobenius norm difference between the estimated and true Laplacian. 
\begin{equation}
    \begin{aligned}
        \label{Relative Error}
        \mathrm{RE} = \frac{\|\hat{\alphavec}-\alphavec_{\text{true}\|_{2}}}{\|\alphavec_{\text{true}}\|_{2}}.
    \end{aligned}
\end{equation}
We compare the performance of the following algorithms: 
\renewcommand{\labelenumi}{\arabic{enumi}.}
\begin{enumerate}
    \item \textbf{\Ac{pgd}} Approach: Implemented by solving Equation 9 in \cite{Yakov2023Laplace}, using a regularization parameter $\lambda = 0.2$. 
    \item \textbf{\ac{newgle}} Approach \cite{Yakov2023Laplace}: A proximal Newton approach for \ac{lgmrf} estimation, incorporating an MCP penalty and an inner projected nonlinear conjugate solver. 
    \item \textbf{\ac{alpe}} Approach \cite{ying2021minimax}: A wieghted $\ell_{1}$-norm regularized \ac{mle} approach where the weights are computed adaptively.  
\end{enumerate}
We compare these methods against the \acp{crb} of Theorem \ref{Theorem2 - Sparse Graph}.

Figure \ref{fig: Logarithmic Scale} presents the relative error across different sample size ratios $(\frac{n}{p})$. As expected, the relative error consistently decreases with an increasing number of observations ($n$) relative to the number of nodes ($p$), underscoring the importance of sufficient data for accurate estimation. Among the tested algorithms (\ac{pgd}, \ac{newgle}, \ac{alpe}), the \ac{pgd} algorithm consistently achieves the lowest relative error. Furthermore, comparing these practical algorithms to the theoretical oracle \acp{crb}, $\Bmat_{1}$ and $\Bmat_{2}$, reveals that the gap between actual algorithm performance and the theoretical lower bounds narrows significantly with larger sample sizes. 
Overall, these results illustrate the effectiveness of \ac{pgd} in Laplacian estimation and highlight the utility of the proposed \acp{crb} as benchmarks for evaluating and developing future estimation methodologies. Furthermore, it can be seen that the bound $\Bmat_{2}$ is consistently tighter than $\Bmat_{1}$ across the entire range of sample size ratios. 
For $\frac{n}{p}>35$, all estimators closely approach the oracle bounds, indicating the support is successfully recovered in these regimes and that the estimators achieve near-optimal performance.

\begin{figure}[hbt]
    \centering
\includegraphics[width=0.8\linewidth]{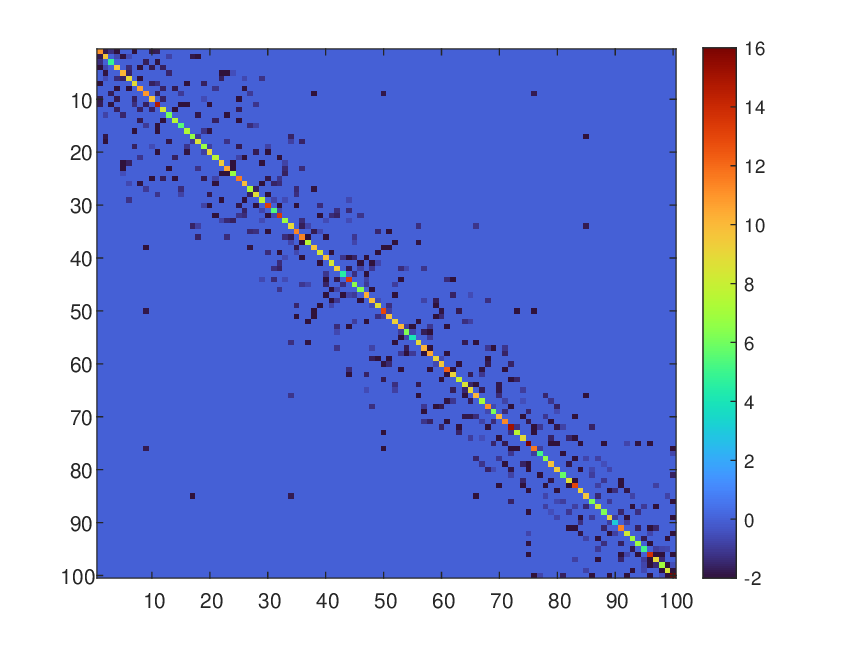}
    \caption{Random planar graph with $100$ nodes used in the simulations of Subsection \ref{simulation_GMRF_sec}.}
    \label{fig: Laplacian Matrix GMRF 100 Nodes}
\end{figure}

\begin{figure}[hbt]
    \centering
\includegraphics[width=0.8\linewidth]{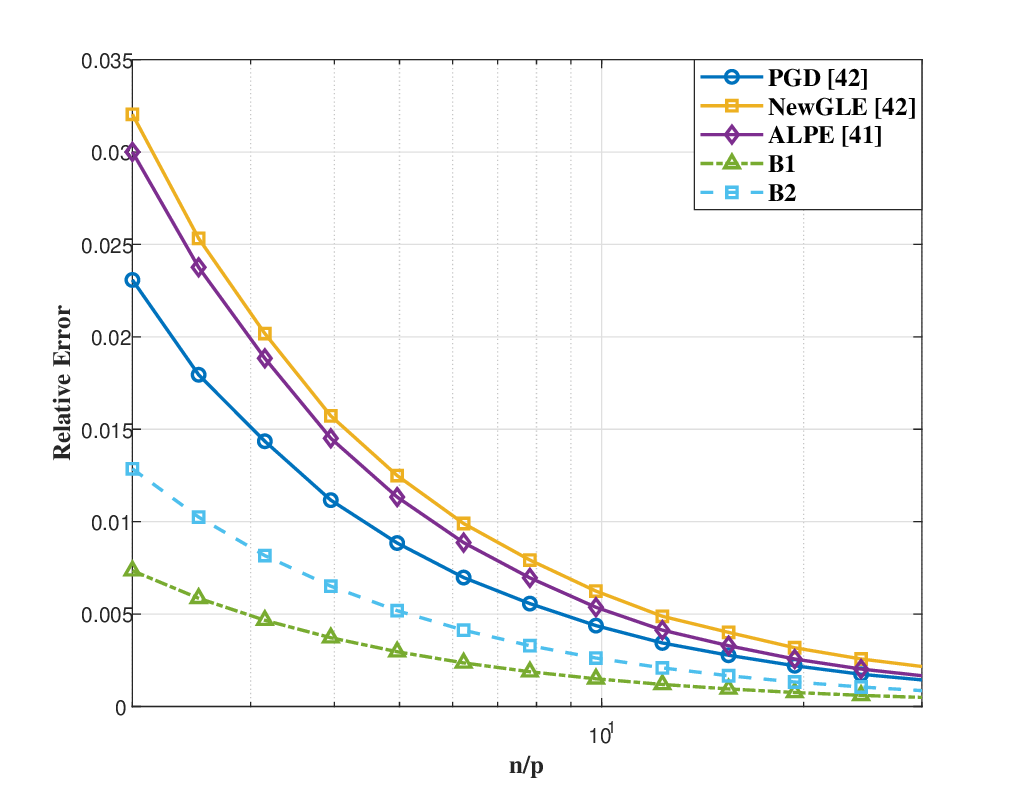}
    \caption{The \ac{mse} of the different estimators compared with the oracle \acp{crb} for planner graph versus $n/p$.}
    \label{fig: Logarithmic Scale} 
\end{figure}


\section{Conclusion}
\label{sec:conclusions}
In this paper, we have investigated the problem of developing performance bounds for the estimation of Laplacian matrices under structural and sparsity constraints, which are fundamental in signal processing problems involving networks. We have demonstrated that the estimation problem under Laplacian structural constraints can be represented via a linear reparameterization that enforces symmetry and the null-space condition. By leveraging a linear reparametrization that inherently enforces the symmetry and null-space properties of the Laplacian matrix, we derived a reduced-dimensional FIM and formulated closed-formed \acp{crb} for both fully connected and sparse graphs.
For sparse Laplacian matrices, we introduced two oracle \acp{crb} that incorporate knowledge of the true support set.
 We analyzed the order relations between the bounds and provided an associated Slepian-Bangs formula for the Gaussian case.
 The oracle \acp{crb} and the order relation between them are applicable to general sparse recovery problems, making the discussion valuable for a broad class of statistical inference tasks.
We applied the \acp{crb} to three representative applications:  (i) power system topology estimation via admittance matrix recovery, (ii) graph filter identification in GSP, and (iii) precision matrix estimation in GMRFs under Laplacian constraints.
 Our simulation results demonstrated that the \ac{cmle} asymptotically achieves the oracle \ac{crb}.
 Additionally, the CRBs accurately characterize estimators' performance across different noise and sample-size regimes, and help to identify the regions where the inequality constraints are naturally satisfied. 
The results show that the proposed bound can be used as a benchmark for the estimation performance of the
CMLE and oracle CMLE, and aids in investigating the influence of the Laplacian constraints on the estimation performance.



Future work can extend the results in several directions, including the extension to complex-valued Laplacian matrices, which naturally arise in power systems and communication networks \cite{halihal2024estimation}. 
Additionally, an interesting research avenue is the design of sampling strategies based on the \ac{crb} \cite{Lital2023Smooth}.  Finally, an important direction is the development of blind estimation techniques for scenarios where some network measurements are missing or incomplete.


\bibliographystyle{IEEEtran}


\end{document}